\documentclass{article}

\usepackage{microtype}
\usepackage{graphicx}
\usepackage{subfigure}
\usepackage{booktabs} 
\usepackage{algorithm}
\usepackage{algorithmic}
\usepackage[hidelinks]{hyperref}
\usepackage{url}
\usepackage{dsfont} 
\usepackage{siunitx}
\usepackage{amsmath,amssymb,amsfonts}
\usepackage{amsthm}
\usepackage{cleveref}
\usepackage{mathrsfs}
\usepackage{hyperref}

\usepackage[accepted]{icml2022}

\newtheorem{assumption}{Assumption}
\newtheorem{lemma}{Lemma}
\newtheorem{theorem}{Theorem}
\newtheorem{corollary}{Corollary}
\icmltitlerunning{Risk-Averse No-Regret Learning in Online Convex Games}

\begin{document}

\twocolumn[
\icmltitle{Risk-Averse No-Regret Learning in Online Convex Games}



\icmlsetsymbol{equal}{*}




\begin{icmlauthorlist}
\icmlauthor{Zifan Wang}{kth}
\icmlauthor{Yi Shen}{duke}
\icmlauthor{Michael M. Zavlanos}{duke}
\end{icmlauthorlist}
\icmlaffiliation{kth}{School of Electrical Engineering and Computer Science, KTH Royal Institute of Technology, Stockholm, Sweden.}
\icmlaffiliation{duke}{Department of Mechanical Engineering \& Material Science, Duke University, Durham, NC 27708, USA}

\icmlcorrespondingauthor{Zifan Wang}{zifanwang199710@gmail.com}

\icmlkeywords{Machine Learning, ICML}

\vskip 0.3in
]



\printAffiliationsAndNotice{}  

\begin{abstract}
We consider an online stochastic game with risk-averse agents whose goal is to learn optimal decisions that minimize the risk of incurring significantly high costs. Specifically, we use the Conditional Value at Risk (CVaR) as a risk measure that the agents can estimate using bandit feedback in the form of the cost values of only their selected actions.
Since the distributions of the cost functions depend on the actions of all agents that are generally unobservable, they are themselves unknown and, therefore, the CVaR values of the costs are difficult to compute.
To address this challenge, we propose a new online risk-averse learning algorithm that relies on one-point zeroth-order estimation of the CVaR gradients computed using CVaR values that are estimated by appropriately sampling the cost functions.
We show that this algorithm achieves sub-linear regret with high probability. 
We also propose two variants of this algorithm that improve performance. The first variant relies on a new sampling strategy that uses samples from the previous iteration to improve the estimation accuracy of the CVaR values. The second variant employs residual feedback that uses CVaR values from the previous iteration to reduce the variance of the CVaR gradient estimates. We theoretically analyze the convergence properties of these variants and illustrate their performance on an online market problem that we model as a Cournot game. 
\end{abstract}

\section{Introduction}
Online convex optimization (OCO) aims at solving optimization problems with unknown cost functions using only samples of the cost function values.
Many practical applications can be modeled as OCO problems. Examples include spam filtering \cite{hazan2019introduction} and portfolio management \cite{hazan2006efficient}, among many others \cite{shalev2011online}. Oftentimes, OCO problems involve multiple agents interacting with each other in the same environment; for instance, in traffic routing \cite{sessa2019no} and economic market optimization \cite{shi2019no}, agents cooperate or compete, respectively, by sequentially selecting the best decisions that minimize their expected accumulated costs. These problems can be formulated as online convex games \cite{shalev2006convex,gordon2008no}, and constitute the focus of this paper.

Typically, the performance of online optimization algorithms is measured using different notions of regret \cite{hazan2019introduction}, that capture the difference between the agents' online decisions and the optimal decisions in hindsight.
An online algorithm is said to be no-regret (no-external-regret) if its regret is sub-linear in time \cite{gordon2008no}, i.e., if the agents are able to eventually learn the optimal decisions. 
Many no-regret algorithms have been proposed and analyzed for online convex games including \cite{shalev2006convex,gordon2008no,hazan2019introduction,shalev2011online}.
Common in these problems is the objective of the agents to minimize their expected cost functions. 
However, in high-stakes applications, minimizing the expected cost alone is not sufficient; avoiding the worst case is equally important. For example, in portfolio management, investing in the assets that yield 
the highest expected return rate is not necessarily the best decision since these assets may also be highly volatile and result in severe losses. To control for such catastrophic events, appropriate risk-averse criteria need to be considered during optimization, such as the Sharpe Ratio \cite{sharpe1994sharpe} or Conditional Value at Risk (CVaR) \cite{artzner1999coherent}.  

In this paper, we consider online convex games with risk-averse agents, whose goal is to minimize the CVaR values of their cost functions. Moreover, we assume that only bandit feedback in the form of the costs of selected actions is available to the agents to estimate their CVaR values.
To the best of our knowledge, risk-averse learning in convex games has not been explored in the literature. Most closely related to the problem considered here is work on deterministic online convex games with bandit feedback, including \cite{bravo2018bandit,duvocelle2018learning,tatarenko2018learning,lin2020finite}. Specifically, \cite{bravo2018bandit} relies on tools from stochastic approximation theory to show that derivative-free methods for monotone and concave games converge to the Nash equilibrium with probability $1$. This work is extended in \cite{duvocelle2018learning} to time-varying games.
The authors in \cite{duvocelle2018learning} show that if the time-varying game converges, then the sequence of actions converges to the Nash equilibrium.
Common in these works is that the cost functions are deterministic.
As such, they cannot model risk in the presence of uncertainty. Methods for risk-averse learning have been investigated, e.g., in \cite{urpi2021risk,kalogerias2019zeroth,chow2017risk}.
Specifically, in \cite{urpi2021risk}, a risk-averse offline reinforcement learning algorithm is proposed that exhibits better performance compared to risk-neural approaches for robot control tasks. In \cite{kalogerias2019zeroth}, a zeroth-order method for mean-semideviation-based risk-averse learning is proposed.
We note that, despite the importance of controlling risk in many applications, only a few works employ CVaR as a risk measure and still provide theoretical results, e.g., \cite{curi2019adaptive,cardoso2019risk,tamkin2019distributionally,soma2020statistical,kalogerias2020noisy}. 
In \cite{curi2019adaptive}, risk-averse learning is transformed into a zero-sum game between a sampler and a learner. Then, using an adaptive sampling strategy, the regret of this game is analyzed. In \cite{tamkin2019distributionally}, a sub-linear regret algorithm is proposed for risk-averse multi-arm bandit problems by constructing empirical cumulative distribution functions for each arm from online samples. 
{Recently,  \cite{kalogerias2020noisy} has shown that CVaR learning problems subject to not necessarily convex loss functions can be solved as efficiently as their risk-neutral counterparts.}

Compared to the literature discussed above, risk-averse learning for online convex games possesses unique challenges, including: (1)
The distribution of an agent's cost function depends on other agents' actions, and
(2) Using finite bandit feedback, it is difficult to accurately estimate the continuous distributions of the cost functions and, therefore, accurately estimate the CVaR values. 
To address these challenges, in this paper we use samples of the cost functions to learn an empirical distribution function (EDF) of the random costs.
Then, using this EDF, the agents can estimate the CVaR values of their cost functions, and use these CVaR values to construct zeroth-order estimates of the CVaR gradients. By appropriately designing this sampling strategy, we show that with high probability, the accumulated error of the CVaR estimates is bounded, and the accumulated error of the zeroth-order CVaR gradient estimates is also bounded. As a result, our method achieves sub-linear regret with high probability.
To further improve the regret of our method, we allow our sampling strategy to use previous samples to reduce the accumulated error of the CVaR estimates.
Specifically, at time step $t$, we build the EDF estimate using samples from times $t$ and $t-1$, and then use this EDF to estimate the CVaR values and the corresponding CVaR gradients, as before. Assuming that the variation of the CDF of the cost function at two consecutive time steps is bounded by the distance between the two corresponding actions at these time steps, 
we theoretically show that the accumulated error of the CVaR estimates is strictly less than that achieved without reusing previous samples under certain conditions. 
We also provide an alternative way of improving the regret  by utilizing residual feedback \cite{zhang2020boosting,zhang2020improving} that reduces the variance of the zeroth-order CVaR gradient estimates. We illustrate our method on an online market problem that we model as a Cournot game \cite{allaz1993cournot}.

To the best of our knowledge, this is the first work to address risk-averse learning in online convex games. 
Note that the CVaR value of each agent depends on the joint actions of all agents, hence the proposed risk-averse game is in essence a time-varying game as the agents update their actions sequentially. All existing literature on learning in games discussed before considers static games, except for \cite{duvocelle2018learning}, that requires knowledge of whether the game converges and how the Nash equilibrium changes. However, the time-varying nature of the game considered here is due to the updates of the other agents and, therefore, it is not possible to know {\emph{a prior}} whether this game will converge or not. 
As a result, the analysis in \cite{duvocelle2018learning} cannot be applied to analyze the game considered here.
In addition, existing literature that employs zeroth-order techniques to solve learning problems in games typically relies on constructing unbiased gradient estimates of the smoothed cost functions. Nevertheless, unbiased gradient estimates cannot be obtained in risk-averse games since it is not possible to obtain accurate  CVaR estimates of the cost functions merely using finite bandit feedback.
Perhaps closest to the method proposed here is the approach in  \cite{cardoso2019risk}, that makes a first attempt to analyze risk-averse bandit learning problems. Using CVaR properties, the authors reformulate the CVaR optimization problem to an equivalent optimization problem of an augmented $L$ function and show that the reformulated problem converges in the single-agent case. 
However, the analysis in \cite{cardoso2019risk} cannot be easily extended to multi-agent problems since minimizing the $L$ functions is not equivalent to minimizing CVaR values in multi-agent games.

The rest of the paper is organized as follows. In Section \ref{sec_prob}, we define the proposed risk-averse online game and provide some assumptions. The main algorithm is presented in Section \ref{sec_main} with corresponding regret analysis. Two variants of this algorithm with improved regrets are provided in Section \ref{sec:improve}. In section \ref{sec_simu}, we use an online market example to illustrate the effectiveness of the proposed algorithms. Finally, we conclude this work in Section \ref{sec_conclu}.

\section{Problem Definition}
\label{sec_prob}
We consider a repeated game $\mathcal{G}$  with $N$ agents. 
At the beginning of each episode, each agent $i\in \mathcal{N}=\{1,\ldots,N\}$ simultaneously chooses an action $x_i \in \mathbb{R}^{d_i}$ from a convex set $\mathcal{X}_i$ and receives a random cost value that is sampled from the cost function $J_i(x_i,x_{-i},\xi_i): \mathcal{X} \times \Xi_i \rightarrow \mathbb{R}$, where $x_{i}$ is the action of agent $i$, $x_{-i}$ denotes the actions of agents except for agent $i$, $\mathcal{X}=\Pi_{i=1}^N\mathcal{X}_i$ is the joint action space and $\xi_i \in \Xi_i$ describes the uncertainty of the cost function. Here we assume that the diameter of the convex set $\mathcal{X}_i$ is bounded by $D_x$ for all $i=1,\ldots,N$. For ease of notation, we sometimes denote the cost function as $J_i(x,\xi_i)$, where $x=(x_i,x_{-i})$ is the concatenated vector of all agents' actions.

We use the Conditional Value at Risk (CVaR) as a risk measure to model the risk-aversion in the agents. 
Specifically, suppose that the random variable 
$J_i(x,\xi_i)$ has the CDF $F_x(y)=\mathbb{P}\{J_i(x,\xi_i)\leq y \}$. We drop the decision variable $x$ in $F_x$ whenever it is clear from the contexts. Then, for a given risk level $\alpha_i\in [0,1]$, the CVaR of the cost function $J_i(x,\xi_i)$ of agent $i$ at point $x$ is defined as
\begin{align*} 
    C_i(x):& ={\rm{CVaR}}_{\alpha_i}[J_i(x,\xi_i)] \\
    & = \mathbb{E}_F[J_i(x,\xi_i)|J_i(x,\xi_i)\geq J^{\alpha_i}],
\end{align*} 
where $J^{\alpha_i}$ is the $1-\alpha_i$ quantile of the distribution, which is also termed as Value at Risk (VaR).
CVaR captures the average cost under the tail of the distribution of $J_i(x,\xi_i)$.
Note that the CVaR value of the random variable $J_i(x,\xi_i)$ is determined its cumulative distribution function $F$, hence we sometimes write CVaR as a function of the CDF, i.e.,  ${\rm{CVaR}}_{\alpha_i}[J_i(x,\xi_i)]= {\rm{CVaR}}_{\alpha_i}[F]$ for ease of notation.

Here we assume that the agents have no prior knowledge about this game, i.e, the agents do not know the cost functions $J_i(x,\xi_i)$, and cannot observe the other agents' actions. The only information that is available to the agents is the cost of the selected actions. Moreover, we assume that the cost function $J_i(x,\xi_i)$ of each agent satisfies the following assumptions.
\begin{assumption}\label{assump:J_convex}
The function $J_i(x_i,x_{-i},\xi_i)$ is convex in $x_i$ for every $\xi_i \in \Xi_i$ and bounded  by $U$, i.e., $|J_i(x,\xi_i)|\leq U$, for all $i=1,\ldots,N$.
\end{assumption}
\begin{assumption}\label{assump:J_Lips}
$J_i(x,\xi_i)$ is $L_0$-Lipschitz continuous in $x$ for every $\xi_i \in \Xi_i$, for all $i=1,\ldots,N$. 
\end{assumption}
Assumptions \ref{assump:J_convex} and \ref{assump:J_Lips} hold in many applications, e.g., the Cournot game \cite{shi2019no} and the repeated Kelly auctions \cite{duvocelle2018learning}.

Given Assumptions 1 and 2, the below well known result characterizes important properties of the CVaR function. The proof can be found in \cite{cardoso2019risk}.
\begin{lemma}\label{lemma:cvar_property}
Given Assumptions \ref{assump:J_convex} and \ref{assump:J_Lips}, we have that $C_i(x_i,x_{-i})$ is convex in $x_i$ and $L_0$-Lipschitz continuous in $x$, for all $i=1,\ldots,N$. 
\end{lemma}
The objective of each agent $i$ is to minimize its cumulative CVaR functions.
Specifically, given a sequence of agents' actions $\{\hat{x}_{t} \}_{t=1}^T$ over $T$ episodes, where $\hat{x}_t=(\hat{x}_{i,t},\hat{x}_{-i,t})$ denotes the agents' actual played actions at time step $t$, we define the regret (CVaR-regret) for agent $i$ as
\begin{align*}
    {\rm{R}}_{C_i}(T) = \sum_{t=1}^{T} C_i(\hat{x}_{i,t},\hat{x}_{-i,t}) - \mathop{\rm{min}}_{\tilde{x}_i \in \mathcal{X}_i} \sum_{t=1}^{T} C_i(\tilde{x}_i,\hat{x}_{-i,t}),
\end{align*}
which measures the cumulative loss against a best single policy in hindsight. 
Then, our goal in this paper is to design a no-regret (equivalently, sub-linear regret) algorithm to solve this game, such that  $\lim_{T\rightarrow \infty}\frac{{\rm{R}}_{C_i}(T)}{T} = 0$ for all agents. 

\section{A Risk-Averse Learning Algorithm}\label{sec_main}
In this section, we propose a risk-averse learning algorithm to solve the proposed online convex game. Our algorithm relies on a novel sampling strategy to estimate the CVaR values and a one-point zeroth-order estimator of the CVaR gradient. 
Specifically, since estimation of CVaR values requires the distribution of the cost functions which is impossible to compute using a single evaluation of the cost functions per time step, we assume that the agents can sample the cost functions multiple times to learn their distributions. For this, we introduce a practical sampling strategy described below.

During each time step $t$, the agents keep their actions fixed and draw $n_t$ samples of their individual cost functions. Then, they use these samples to determine their actions for time step $t+1$ and then sample again.
The sampling strategy is defined as
\begin{align}\label{eq:sample_strategy}
    n_t = \lceil b  U^2 (T-t+1)^a \rceil,
\end{align}
where $\lceil \cdot \rceil$ is the ceiling function, $T$ is the time horizon, $U$ is the cost function bound as in Assumption \ref{assump:J_convex}, and $a,b \in(0,1)$ are parameters to be selected later. 
The parameters $a,b$ are assumed to be known by all the agents beforehand so that the game is synchronous. Moreover, since
$a<1$, the number of samples $n_t$ will decrease with the iterations and eventually equal to $1$.

The proposed risk-averse learning algorithm for online convex games is illustrated in Algorithm \ref{alg:algorithm1}. 
At time step $t$, the agents randomly perturb their current actions $x_{i,t}$ by an amount $\delta u_{i,t}$, where $u_{i,t} \in \mathbb{S}^{d_i}$ is a random perturbation direction sampled from a unit sphere $\mathbb{S}^{d_i}\subset \mathbb{R}^{d_i}$ and $\delta$ is the size of this perturbation. 
Then, using the sampling strategy defined above, the agents play their perturbed actions $\hat{x}_{i,t}=x_{i,t}+\delta u_{i,t} $ for $n_t$ times, and obtain $n_t$ samples of their cost functions. 
For agent $i$, at time step $t$, we denote the CDF of the random cost $J_i(\hat{x}_t,\xi_i)$ that is returned by the perturbed action $\hat{x}_t$ as $F_{i,t}(y)= \mathbb{P}\{ J_i(\hat{x}_t,\xi_i) \leq y\}$. 
Since the agents cannot accurately estimate this continuous CDF $F_{i,t}(y)$ using finite samples, they instead construct the EDF of $J_i(\hat{x}_t,\xi_i)$ as 
\begin{align}\label{eq:edf_alg1}
    \hat{F}_{i,t}(y)= \frac{1}{n_t}\sum_{j=1}^{n_t}\mathbf{1}\{ J_i(\hat{x}_t,\xi_i^j) \leq y\},
\end{align}
where $\mathbf{1}\{ \cdot\}$ is the indicator function. Then, using this EDF, the agents construct CVaR estimates of their cost functions $J_i(\hat{x}_t,\xi_i) $, denoted as ${\rm{CVaR}}_{\alpha_i}[\hat{F}_{i,t} ]$, which they use to further construct zeroth-order estimates of their CVaR gradients as
\begin{align}\label{eq:grad_alg1}
    \hat{g}_{i,t}=\frac{d_i}{\delta} {\rm{CVaR}}_{\alpha_i} [\hat{F}_{i,t} ] u_{i,t},
\end{align}
where $\delta$ is the size of the perturbation on the action $x_{i,t}$ defined above. To ensure that the function values at the queried points during each time step are always feasible, we define the projection set $\mathcal{X}_i^{\delta} =\{x_i \in \mathcal{X}_i | {\rm{dist}}(x_i,\partial \mathcal{X}_i) \geq \delta \}$. Then, the agents perform the following projected  gradient-descent update
\begin{align}\label{eq:projection}
    x_{i,t+1} = \mathcal{P}_{\mathcal{X}_i^{\delta}} ( x_{i,t} - \eta \hat{g}_{i,t}).
\end{align}

\begin{algorithm}[t]
\caption{Risk-averse learning } \label{alg:algorithm1}
\begin{algorithmic}[1]
    \REQUIRE Initial value $x_0$, step size $\eta$, parameters $a$, $b$, $\delta$, $T$, risk level $\alpha_i$, $i=1,\cdots,N$.
    \FOR{$episode \; t=1,\ldots,T$}
        \STATE Select $n_t=\lceil b U^2 (T-t+1)^a\rceil$
        \STATE Each agent samples $u_{i,t} \in \mathbb{S}^{d_i}$,  $i=1,\ldots, N$
        \STATE Each agent play $\hat{x}_{i,t}=x_{i,t}+\delta u_{i,t} $, $i=1,\ldots, N$
        \FOR{$j=1,\ldots,n_t$}
        \STATE Let all agents play $\hat{x}_{i,t}$  
        \STATE Obtain $J_i(\hat{x}_{i,t},\hat{x}_{-i,t},\xi_{i}^j)$
        \ENDFOR
        \FOR{agent $ i=1,\ldots,N$}
        \STATE Build EDF $\hat{F}_{i,t}(y)$ 
        \STATE Calculate CVaR estimate: $ {\rm{CVaR}}_{\alpha_i}[\hat{F}_{i,t}] $ 
        \STATE Construct gradient estimate\\ $\hat{g}_{i,t}=\frac{d_i}{\delta} {\rm{CVaR}}_{\alpha_i} [\hat{F}_{i,t} ] u_{i,t}$
        \STATE Update $x$: $x_{i,t+1} \leftarrow \mathcal{P}_{\mathcal{X}_i^{\delta}} ( x_{i,t} - \eta \hat{g}_{i,t})$
        \ENDFOR
    \ENDFOR
\end{algorithmic}
\end{algorithm}

To analyze Algorithm \ref{alg:algorithm1}, we utilize the smoothed approximation of $C_i(x)$ defined as $C_i^{\delta}(x)= \mathbb{E}_{w_i \sim \mathbb{B}_i,u_{-i} \sim \mathbb{S}_{-i}}[C_i(x_i+\delta w_i,x_{-i}+\delta u_{-i})]$, where $\mathbb{S}_{-i}=\Pi_{j \neq i} \mathbb{S}_j$, and $\mathbb{B}_i$, $\mathbb{S}_i$ denote the unit ball and unit sphere in $\mathbb{R}^{d_i}$, respectively, and the size of the perturbation $\delta$ here serves as a smoothing parameter that controls how well $C_i^{\delta}(x)$ approximates $C_i(x)$. For details on zeroth-order optimization methods and their analysis, see \cite{nesterov2017random}.
The function $C_i^{\delta}(x)$ satisfies the following properties. The proof can be found in Appendix A.1.
\begin{lemma}\label{lemma:cdelta_property}
Let Assumptions \ref{assump:J_convex} and \ref{assump:J_Lips} hold. Then we have that
\begin{enumerate}
    \item $C_i^{\delta}(x_i,x_{-i})$ is convex in $x_i$,
    \item $C_i^{\delta}(x)$ is $L_0$-Lipschitz continuous in $x$,
    \item $|C_i^{\delta}(x)$-$C_i(x)|\leq \delta L_0 \sqrt{N}$.
\end{enumerate}
\end{lemma}

From Lemma C.1 in \cite{bravo2018bandit} , we have that 
\begin{align}
    \mathbb{E}[ \frac{d_i}{\delta} C_i(\hat{x}_t) u_{i,t}] = \nabla_i C_i^{\delta}(x_t),
\end{align}
where $\nabla_i$ denotes the partial derivative with respect to $x_i$. 
However, as discussed before, it is not possible to accurately estimate the CVaR value $C_i(\hat{x}_t)$ using finite samples of the cost function $J_i(\hat{x}_t,\xi_i)$. 
Instead, there will exist a CVaR estimation error, which we define as 
\begin{align*}
    \hat{\varepsilon}_{i,t} := 
     {\rm{CVaR}}_{\alpha_i}[\hat{F}_{i,t} ]- {\rm{CVaR}}_{\alpha_i}[F_{i,t} ].
\end{align*} 
Then, we have that  $\mathbb{E}[\hat{g}_{i,t}]=\mathbb{E}\left[ \frac{d_i}{\delta}(C_i(\hat{x}_t)+ \hat{\varepsilon}_{i,t})u_{i,t} \right]= \nabla_i C_i^{\delta}(x_t)+ \mathbb{E}\left[\frac{d_i}{\delta}\hat{\varepsilon}_{i,t}u_{i,t} \right]$, which indicates that the CVaR gradient estimate is biased due to the use of finite samples. The analysis of the CVaR estimation error $\hat{\varepsilon}_{i,t}$ plays a key role in the whole analysis of the regret of Algorithm \ref{alg:algorithm1}. To bound the CVaR estimation error, we first present a lemma that bounds the difference between the CVaR values for two different CDFs. The proof of this lemma can be found in Appendix A.2.
\begin{lemma}\label{lemma:CVaR_bound}
Let $F$ and $G$ be two CDFs of two random variables and the random variables are bounded by $U$. Then we have that
\begin{align*}
    |{\rm{CVaR}}_{\alpha}[F]-{\rm{CVaR}}_{\alpha}[G]| \leq \frac{U}{\alpha}  \mathop{\rm{sup}}_{y} |F(y)-G(y)|.
\end{align*}
\end{lemma}
Lemma \ref{lemma:CVaR_bound} states that the distance between two CVaR values is related to the distance between the corresponding CDFs. By substituting $F=F_{i,t}$ and $G=\hat{F}_{i,t}$ into Lemma \ref{lemma:CVaR_bound} and applying the Dvoretzky–Kiefer–Wolfowitz (DKW) inequality, we have that
\begin{align}\label{eq:epsi_hat}
    |\hat{\varepsilon}_{i,t}|&=|{\rm{CVaR}}_{\alpha_i}[\hat{F}_{i,t} ]- {\rm{CVaR}}_{\alpha_i}[F_{i,t} ]| \nonumber\\
    &\leq \frac{U}{\alpha_i}\sqrt{\frac{ \ln(2/\bar{\gamma})}{2n_t }}
\end{align}
with probability at least $1-\bar{\gamma}$.
Combining inequality \eqref{eq:epsi_hat} with the sampling strategy defined in equation (\ref{eq:sample_strategy}), the accumulated error of CVaR estimation can be bounded, which is given in the following lemma whose proof can be found in Appendix A.3.
\begin{lemma}\label{lemma:sum_epsihat_bound}
Given a confidence level $\bar{\gamma}$ and the sampling strategy in equation (\ref{eq:sample_strategy}), we have that the following inequality holds:
\begin{align}\label{eq:sum_epsi_hat}
    \sum_{t=1}^T |\hat{\varepsilon}_{i,t}| \leq B_1
\end{align}
with probability at least $1-\gamma$, where $\gamma=\bar{\gamma}T$ and $B_1=\frac{1}{\alpha_i}\sqrt{\frac{2\ln(2T/\gamma)}{ b }}T^{1-\frac{a}{2}}$.
\end{lemma}

The following result provides a generic regret decomposition of Algorithm \ref{alg:algorithm1}. The proof can be found in Appendix A.4.
\begin{lemma}\label{lemma:regret_decomp}
Let Assumptions \ref{assump:J_convex} and \ref{assump:J_Lips}  hold. Then, the regret of Algorithm \ref{alg:algorithm1} satisfies 
\begin{align*}
    {\rm{R}}_{C_i}^1(T) \leq Err(\text{ZO})+ Err(\text{CVaR}),
\end{align*}
where $ Err(\text{ZO})= \frac{D_x^2}{2\eta}  +\frac{d_i^2 U^2 \eta}{2\delta^2}T   + (4 \sqrt{N}+\Omega)L_0\delta T$, $Err(\text{CVaR})= \frac{d_i D_x }{\delta} B_1$, $\Omega>0$ is a constant that represents the error from projection $\mathcal{P}$, and $B_1$ as in (\ref{eq:sum_epsi_hat}).
\end{lemma}
Lemma \ref{lemma:regret_decomp} decomposes the regret into two terms, a zeroth-order error term and a CVaR estimation error term. By selecting $\eta$ and $\delta$ appropriately, we can show that Algorithm \ref{alg:algorithm1} is no-regret. In the following theorem, $\tilde{\mathcal{O}}$ hides constant factors and poly-logarithmic factors of $T$. In contrast, the standard notation $\mathcal{O}$ only hides constant factors.
\begin{theorem}\label{thm:1}
Let Assumptions \ref{assump:J_convex} and \ref{assump:J_Lips} hold and select $\delta=\frac{\sqrt{D_x U d_i}}{N^{\frac{1}{4}} T^{\frac{a}{4}} \sqrt{\alpha_i L_0} }$,  $\eta=\frac{\sqrt{\alpha_i} D_x^{\frac{3}{2}}}{\sqrt{L_0 U d_i} N^{\frac{1}{4}} T^{\frac{3a}{4}} }$. 
Suppose that $n_t$ is chosen as in equation (\ref{eq:sample_strategy}) with $a\in(0,1)$, and the EDF and the gradient estimate are defined as in equations (\ref{eq:edf_alg1}) and (\ref{eq:grad_alg1}), respectively. Then, Algorithm \ref{alg:algorithm1} achieves regret ${\rm{R}}_{C_i}^1(T) =\tilde{\mathcal{O}} (T^{1-\frac{a}{4}})$ with probability at least $1-\gamma$ .
\end{theorem}
\begin{proof}
Substituting $\delta$, $\eta$ and $B_1$ into the regret bound $Err(\text{ZO})+Err(\text{CVaR})$ in Lemma \ref{lemma:regret_decomp}, we obtain that ${\rm{R}}_{C_i}^1(T)\leq Err(\text{ZO})+Err(\text{CVaR})=\mathcal{O}( \sqrt{D_x U d_i L_0}N^{\frac{1}{4}}\alpha_i^{-\frac{1}{2}} \sqrt{\ln(T/\gamma)} T^{1-\frac{a}{4}} )= \tilde{\mathcal{O}} (T^{1-\frac{a}{4}})$. The proof is complete.
\end{proof}
As shown in Theorem \ref{thm:1}, although it is impossible to obtain accurate CVaR values using finite bandit feedback, our method still achieves sub-linear regret with high probability. Notice that the choice of the risk level $\alpha_i$ can also affect the regret. Specifically, a lower value for $\alpha_i$ can result in higher regret, since it is harder to get samples under the $\alpha_i$ tail of the distribution.

\section{Improving the Algorithm Regret}\label{sec:improve}
In this section, we propose two variants of Algorithm \ref{alg:algorithm1} that improve the regret. The first variant reduces the accumulated error of the CVaR estimates by using samples from the previous iteration. The second variant employs residual feedback \cite{zhang2020boosting,zhang2020improving} to reduce the variance of the CVaR gradient estimates. A relevant analysis is given in the following two subsections, respectively. 
\subsection{Improving the CVaR Estimation Accuracy}\label{sec_alg2}
The accuracy of the CVaR estimation in Algorithm \ref{alg:algorithm1} depends on the number of samples of the cost functions at each iteration according to equation \eqref{eq:epsi_hat}; the more samples, the better the CVaR estimation accuracy. 
To further improve the CVaR estimation accuracy, we propose a modification to Algorithm \ref{alg:algorithm1} that reuses samples from the previous iteration, effectively increasing the number of available samples per iteration while maintaining the number of new samples the same. First, we make the following assumption on the variation of the cumulative distribution function. 
\begin{assumption}\label{assump_F_lips}
Let $F_{i,t}(y)= \mathbb{P}\{ J_i(\hat{x}_t,\xi_i)\leq y\}$ and $F_{i,t-1}(y)=\mathbb{P}\{ J_i(\hat{x}_{t-1},\xi_i)\leq y\}$. There exist constants $C_1,C_2>0$ such that
\begin{align*}
    \mathop{\rm{sup}}_{y} |F_{i,t}(y)-F_{i,t-1}(y)| \leq (C_1 \delta+C_2) \left\|x_t -x_{t-1} \right\|.
\end{align*}
\end{assumption}
Assumption \ref{assump_F_lips} states that the variation of the CDF across two consecutive time steps is bounded by the distance between the corresponding unperturbed actions. It means that if $x_t$ is close to $x_{t-1}$, then the corresponding cost function values $J_i(\hat{x}_t,\xi_i)$ and $J_i(\hat{x}_{t-1},\xi_i)$ for every $\xi_i$ should also be close to each other, and so should be the two CDFs. Note that the bound in Assumption \ref{assump_F_lips} is also related to the smoothing parameter $\delta$, since the played action is in fact the perturbed one, i.e., $\hat{x}_t$. Moreover, note that this bound cannot go to $0$ by decreasing the smoothing parameter $\delta$, which implies that the variation in the CDF should be dominated by the distance between $x_t$ and $x_{t-1}$.

The proposed risk-averse learning algorithm with sample reuse is illustrated in Algorithm \ref{alg:algorithm2} that can be found in Appendix B. 
Specifically, assuming that the agents can sample the cost functions $n_t$ times at every time step $t$, 
for $t\geq2$, we define a new EDF as
\begin{align}\label{eq:edf_alg2}
    \tilde{F}_{i,t}(y)=
    \frac{n_t}{N_t} \hat{F}_{i,t} + \frac{n_{t-1}}{N_t} \hat{F}_{i,t-1},
\end{align}
where $N_t=n_t+n_{t-1}$.
For $t=1$, we set the initial value as $\tilde{F}_{i,1}=\hat{F}_{i,1}$ and $N_1=n_1$. Using this sampling strategy, we design the CVaR gradient estimate as
\begin{align}\label{eq:grad_alg2}
    \tilde{g}_{i,t}=\frac{d_i}{\delta} {\rm{CVaR}}_{\alpha_i} [\tilde{F}_{i,t} ] u_{i,t}.
\end{align}

Note that, as in Algorithm \ref{alg:algorithm1}, this gradient estimate is biased since the estimation of CVaR uses not only finite samples, but also samples from the previous iteration. We denote the CVaR estimation error as 
\begin{align}\label{eq:def_hat_varep}
    \tilde{\varepsilon}_{i,t} := {\rm{CVaR}}_{\alpha_i}[\tilde{F}_{i,t} ] -{\rm{CVaR}}_{\alpha_i}[F_{i,t} ].
\end{align}
Due to the use of previous samples, the analysis of the CVaR estimation error in this case becomes more complicated. The following lemma characterizes the CVaR estimation error and its the proof can be found in Appendix B.1.
\begin{lemma}\label{lemma:Ftilde_bound}
Given a confidence level $\gamma$, the following inequality holds
\begin{align}\label{eq:epsi_tilde_bound}
    |\tilde{\varepsilon}_{i,t}| \leq & \frac{U}{\alpha_i}\left( \sqrt{\frac{ \ln(2T/\gamma)}{2(n_t+n_{t-1}) }}\right)\nonumber\\
    &+\frac{U}{\alpha_i}\left( \frac{(C_1\delta+C_2) d_i U\sqrt{N}\eta}{2\delta} \right),
\end{align}
with probability at least $1-\gamma$, for $\forall t=2,\ldots,T.$
\end{lemma}
Using the above bound on the CVaR estimation error $\tilde{\varepsilon}_{i,t}$, we are able to show the following result.

\begin{lemma}\label{lemma:sum_less}
Assume the same values for $\delta$ and $\eta$ as in Algorithm \ref{alg:algorithm1}, i.e., $\delta=\frac{\sqrt{D_x U d_i}}{N^{\frac{1}{4}}  \sqrt{\alpha_i L_0} }T^{-\frac{a}{4}}$, and $\eta=\frac{\sqrt{\alpha_i} D_x^{\frac{3}{2}}}{\sqrt{L_0 U d_i} N^{\frac{1}{4}}  }T^{-\frac{3a}{4}}$. Then, given any constant $\lambda>0$, there exists $T_{\lambda}>0$ such that when $T>T_{\lambda}$, we have $\sum_{t=1}^T|\tilde{\varepsilon}_{i,t}|\leq B_1- \lambda T^{1-\frac{3a}{4}} $, with probability at least $1-\gamma$.
\end{lemma}
\begin{proof}
\vspace{-0.1in}
For ease of notation, we let $\delta=\Sigma_{1}T^{-\frac{a}{4}}$ and $\eta=\Sigma_2 T^{-\frac{3a}{4}}$, where $\Sigma_1= \frac{\sqrt{D_x U d_i}}{N^{\frac{1}{4}}  \sqrt{\alpha_i L_0} }$, $\Sigma_2=\frac{\sqrt{\alpha_i} D_x^{\frac{3}{2}}}{\sqrt{L_0 U d_i} N^{\frac{1}{4}}  }$.
Summing equation \eqref{eq:epsi_tilde_bound} over $t$, we obtain
\begin{align*}
    & \sum_{t=2}^T |\tilde{\varepsilon}_{i,t}| \nonumber \\
    &\leq \frac{U}{\alpha_i} \sum_{t=2}^T \left( \sqrt{\frac{ \ln(2T/\gamma)}{2(n_t+n_{t-1}) }} + \frac{(C_1\delta+C_2) d_i U\sqrt{N}\eta}{2\delta} \right)  \nonumber \\
    & \leq \frac{U}{\alpha_i} \sum_{t=2}^T  \sqrt{\frac{ \ln(2T/\gamma)}{4n_{t-1} }}+ \frac{(C_1\delta+C_2) d_i U^2 \sqrt{N} \eta T}{2\alpha_i\delta}  \nonumber  \\
    & \leq \sqrt{\frac{ \ln(2T/\gamma)}{\alpha_i^2 b}}T^{1-\frac{a}{2}} + \frac{C_1 d_i U^2 \sqrt{N} \Sigma_2}{2\alpha_i}T^{1-\frac{3a}{4}} \nonumber \\
    & \quad + \frac{C_2d_i U^2\sqrt{N}\Sigma_2}{2\alpha_i \Sigma_1}T^{1-\frac{a}{2}}.  \nonumber 
\end{align*}
Adding the first term $|\tilde{\varepsilon}_{i,1}|$ to both sides of this inequality, we have that
\begin{align}
    & \sum_{t=1}^T |\tilde{\varepsilon}_{i,t}| -B_1 \nonumber \\
    &\leq  -(\sqrt{2}-1)\sqrt{\frac{ \ln(2T/\gamma)} {\alpha_i^2 b}}T^{1-\frac{a}{2}} +\Sigma_3 T^{1-\frac{3a}{4}}\nonumber \\
    &\quad +\Sigma_4 T^{1-\frac{a}{2}} + |\tilde{\varepsilon}_{i,1}|  \nonumber \\
    &\leq \left( -(\sqrt{2}-1)\sqrt{\frac{ \ln(2T/\gamma)} {\alpha_i^2 b}} + \Sigma_4 \right)T^{1-\frac{a}{2}} \nonumber \\
    &\quad +\left(\Sigma_3+|\tilde{\varepsilon}_{i,1}|\right) T^{1-\frac{3a}{4}} \nonumber \\
    &\leq f(T)T^{1-\frac{3a}{4}},
\end{align}
where $\Sigma_3=\frac{C_1 d_i U^2 \sqrt{N} \Sigma_2}{2\alpha_i}$ and  $\Sigma_4=\frac{C_2d_i U^2\sqrt{N}\Sigma_2}{2\alpha_i \Sigma_1}$ and $f(T):=\Sigma_3+|\tilde{\varepsilon}_{i,1}|- \left( (\sqrt{2}-1)\sqrt{\frac{ \ln(2T/\gamma)} {\alpha_i^2 b}} - \Sigma_4 \right)  T^{\frac{a}{4}}$.
Observe that the function $f(T)$ is monotonically decreasing in $T$ and approaches negative infinity when $T\rightarrow \infty$. Hence there exists $T_{\lambda}$ with $f(T_{\lambda})=-\lambda$ such that when $T>T_{\lambda}$ we have $f(T)<-\lambda$, and thus $ \sum_{t=1}^T |\tilde{\varepsilon}_{i,t}| -B_1\leq -\lambda T^{1-\frac{3a}{4}}$.
The proof is complete.
\end{proof}
This lemma shows that for the same confidence level $1-\gamma$, selecting $\delta=\Sigma_{1}T^{-\frac{a}{4}}$ and $\eta=\Sigma_2 T^{-\frac{3a}{4}}$ in Algorithm \ref{alg:algorithm2} results in a bound on the accumulated CVaR estimation error that is strictly less than that achieved by Algorithm \ref{alg:algorithm1}. The proof can be found in Appendix B.2.
Similar to Lemma \ref{lemma:regret_decomp}, we can decompose the regret into two sources of errors as shown below.
\begin{lemma}\label{lemma:regret_decomp_2}
Let Assumptions \ref{assump:J_convex} and \ref{assump:J_Lips}  hold. Then, the regret of Algorithm \ref{alg:algorithm2} satisfies 
\begin{align}
    {\rm{R}}_{C_i}^2(T)
    \leq Err(\text{ZO})+\frac{d_i D_x }{\delta}\sum_{t=1}^T |\tilde{\varepsilon}_{i,t}|,
\end{align}
where $Err(\text{ZO})$ is the zeroth order error term as in Lemma \ref{lemma:regret_decomp}.
\end{lemma}
Recall that the regret achieved by Algorithm \ref{alg:algorithm1} is bounded by $Err(\text{ZO})+ Err(\text{CVaR})= \tilde{\mathcal{O}} (T^{1-\frac{a}{4}})$. In what follows, we show that the regret achieved by Algorithm \ref{alg:algorithm2} is strictly smaller than the regret bound achieved by Algorithm \ref{alg:algorithm1}, i.e., $Err(\text{ZO})+ Err(\text{CVaR})$.
\begin{theorem}\label{thm:2}
Let Assumptions \ref{assump:J_convex}, \ref{assump:J_Lips} and \ref{assump_F_lips} hold, and assume the same values for $\delta$ and $\eta$ as in Algorithm \ref{alg:algorithm1}, i.e., $\delta=\frac{\sqrt{D_x U d_i}}{N^{\frac{1}{4}}  \sqrt{\alpha_i L_0} }T^{-\frac{a}{4}}$ and $\eta=\frac{\sqrt{\alpha_i} D_x^{\frac{3}{2}}}{\sqrt{L_0 U d_i} N^{\frac{1}{4}}  }T^{-\frac{3a}{4}}$.
Suppose that $n_t$ is chosen according to equation (\ref{eq:sample_strategy}) with $a\in(0,1)$, and the EDF and the gradient estimate are defined as in equations (\ref{eq:edf_alg2}) and (\ref{eq:grad_alg2}), respectively. Then, when $T>T_{\lambda}$ with $T_{\lambda}$ as in Lemma \ref{lemma:sum_less}, Algorithm \ref{alg:algorithm2} achieves regret  ${\rm{R}}^2_{C_i}(T)< Err(\text{ZO})+ Err(\text{CVaR})$ and thus ${\rm{R}}^2_{C_i}(T)= \tilde{\mathcal{O}} (T^{1-\frac{a}{4}})$.
\end{theorem}
\begin{proof}
Recall that $Err(\text{CVaR})= \frac{d_i D_x }{\delta} B_1$. Adding and subtracting $Err(\text{CVaR})$ to the bound in Lemma \ref{lemma:regret_decomp_2}, we have that
\begin{align*}
    \quad& {\rm{R}}_{C_i}^2(T) \leq Err(\text{ZO})+\frac{d_i D_x }{\delta}\sum_{t=1}^T |\tilde{\varepsilon}_{i,t}| \nonumber \\
    &= Err(\text{ZO})+ Err(\text{CVaR})+\frac{d_i D_x}{\delta}(\sum_{t=1}^T|\tilde{\varepsilon}_{i,t}|-B_1).
\end{align*}
Combining this inequality with Lemma \ref{lemma:sum_less}, and assuming that $T>T_{\lambda}$, we obtain that
\begin{align*}
    {\rm{R}}_{C_i}^2(T)&\leq Err(\text{ZO})+ Err(\text{CVaR})\nonumber \\
    &=\mathcal{O}( \sqrt{D_x U d_i  L_0}N^{\frac{1}{4}}\alpha_i^{-\frac{1}{2}} \sqrt{\ln(T/\gamma)} T^{1-\frac{a}{4}} ),
\end{align*}
which completes the proof.
\end{proof}
Theorem \ref{thm:2} shows that when Assumption \ref{assump_F_lips} holds, the regret bound achieved by using previous samples is guaranteed to be smaller than that achieved without using prior information. 

Note that hybrid sampling strategies are also possible that initially use samples only from the current iteration and eventually switch to using samples from the previous iteration too. Assuming that $t_0$ denotes the switching time after which previous samples are reused, the EDF is defined as
\begin{align}\label{eq:edf2_alg2}
    \tilde{F}_{i,t}(y)=\left\{ \begin{array}{lc} \hat{F}_{i,t}, & t< t_0+1 \\
    \frac{n_t}{N_t} \hat{F}_{i,t} + \frac{n_{t-1}}{N_t} \hat{F}_{i,t-1}, & t\geq t_0+1\\ \end{array} \right.,
\end{align}
Then, the regret achieved by this hybrid sampling strategy is analyzed below. The proof can be found in Appendix B.3.
\begin{corollary}
Let Assumptions \ref{assump:J_convex}, \ref{assump:J_Lips} and \ref{assump_F_lips} hold, and assume the same values for $\delta$ and $\eta$ as in Algorithm \ref{alg:algorithm1}, i.e., $\delta=\frac{\sqrt{D_x U d_i}}{N^{\frac{1}{4}}  \sqrt{\alpha_i L_0} }T^{-\frac{a}{4}}$ and $\eta=\frac{\sqrt{\alpha_i} D_x^{\frac{3}{2}}}{\sqrt{L_0 U d_i} N^{\frac{1}{4}}  }T^{-\frac{3a}{4}}$.
Suppose that $n_t$ is chosen according to equation (\ref{eq:sample_strategy}) with $a\in(0,1)$, and the EDF and the gradient estimate are defined as in equations (\ref{eq:edf2_alg2}) and (\ref{eq:grad_alg2}), respectively. Then, there exists $T_{\lambda(t_0)}>0$ such that when $T>T_{\lambda(t_0)}$, Algorithm \ref{alg:algorithm2} achieves regret ${\rm{R}}^2_{C_i}(T)<  Err(\text{ZO})+ Err(\text{CVaR})$ and thus ${\rm{R}}^2_{C_i}(T) = \tilde{\mathcal{O}} (T^{1-\frac{a}{4}})$.
\end{corollary}

\subsection{Reducing the CVaR Gradient Estimation Variance}\label{sec_alg3}
Algorithm \ref{alg:algorithm2} discussed in Section \ref{sec_alg2} reduces the CVaR estimation error by using samples from the previous iteration. In this section, we propose an alternative variation of Algorithm \ref{alg:algorithm1} that uses residual feedback \cite{zhang2020boosting,zhang2020improving} to reduce the variance of the zeroth-order CVaR gradient estimates.
The risk-averse learning algorithm with residual feedback is illustrated in Algorithm \ref{alg:algorithm3} and can be found in Appendix C.
Specifically, in Algorithm \ref{alg:algorithm3}, the EDF is computed the same way as in Algorithm \ref{alg:algorithm1},
but the gradient estimate now takes the form
\begin{align}\label{eq:grad_alg3}
    \bar{g}_{i,t}  =\frac{d_i}{\delta} ({\rm{CVaR}}_{\alpha_i} [\hat{F}_{i,t} ]- {\rm{CVaR}}_{\alpha_i} [\hat{F}_{i,t-1} ])  u_{i,t}.
\end{align}
The CVaR gradient estimation is still biased and we can define the error $
    \bar{\varepsilon}_{i,t}= {\rm{CVaR}}_{\alpha_i} [\hat{F}_{i,t} ]- {\rm{CVaR}}_{\alpha_i} [\hat{F}_{i,t-1} ]- {\rm{CVaR}}_{\alpha_i} [F_{i,t} ]$.
Recall the definition of $\hat{\varepsilon}_{i,t}$ in equation \eqref{eq:def_hat_varep}. Substituting in the expression for $\bar{\varepsilon}_{i,t}$, we have that $\bar{\varepsilon}_{i,t}=\hat{\varepsilon}_{i,t} - {\rm{CVaR}}_{\alpha_i} [\hat{F}_{i,t-1} ]$.
Taking the expectation of the gradient in equation \eqref{eq:grad_alg3} with respect to $u_{i,t}$, we have that
\begin{align}\label{eq:bar_git}
    &\mathbb{E}\left[\bar{g}_{i,t}\right] \nonumber =\mathbb{E}[ \frac{d_i}{\delta}(\bar{\varepsilon}_{i,t} + {\rm{CVaR}}_{\alpha_i} [F_{i,t} ] )u_{i,t}] \nonumber \\
    &= \nabla_i C_i^{\delta}(x_t)+ \mathbb{E}\left[\frac{d_i}{\delta}\bar{\varepsilon}_{i,t}u_{i,t} \right]\nonumber \\
    &= \nabla_i C_i^{\delta}(x_t)+ \mathbb{E}\left[\frac{d_i}{\delta}\left( \hat{\varepsilon}_{i,t} -  {\rm{CVaR}}_{\alpha_i} [\hat{F}_{i,t-1} ]\right)u_{i,t} \right]\nonumber \\
    &= \nabla_i C_i^{\delta}(x_t)+ \mathbb{E}\left[\frac{d_i}{\delta} \hat{\varepsilon}_{i,t}u_{i,t}\right] ,
\end{align}
where the last equality is due to the fact that $u_{i,t}$ is independent of $\hat{F}_{i,t-1}$. 
Recall that the gradient estimate in Algorithm \ref{alg:algorithm1} satisfies $\mathbb{E}\left[\hat{g}_{i,t}\right]=\nabla_i C_i^{\delta}(x_t)+ \mathbb{E}\left[\frac{d_i}{\delta} \hat{\varepsilon}_{i,t}u_{i,t}\right]$. Therefore, the expectation of the CVaR gradient estimate using residual feedback is the same as that in Algorithm \ref{alg:algorithm1}.
The following result analyzes the regret achieved by Algorithm \ref{alg:algorithm3}. The proof can be found in Appendix C.
\begin{theorem}\label{thm:3}
Let Assumptions \ref{assump:J_convex} and \ref{assump:J_Lips}  hold, and select $\eta=\frac{D_x  }{d_i L_0 N}T^{-\frac{3a}{4}}$,  $\delta=\frac{D_x }{ N^{\frac{1}{6}}}T^{-\frac{a}{4}}$. 
Suppose that $n_t$ is chosen according to equation (\ref{eq:sample_strategy}) with $a\in(0,1)$, and the EDF and the gradient estimate are defined as in equations (\ref{eq:edf_alg1}) and (\ref{eq:grad_alg3}), respectively. Then, when $T\geq (8N^{\frac{2}{3}} )^{\frac{1}{a}}$, Algorithm \ref{alg:algorithm3} achieves the regret ${\rm{R}}^3_{C_i}(T) =\tilde{\mathcal{O}} (T^{1-\frac{a}{4}})$ with probability at least $1-\gamma$.
\end{theorem}
More precisely, Algorithm \ref{alg:algorithm3} actually achieves the regret  
${\rm{R}}^3_{C_i}(T)= \mathcal{O}( D_x d_i L_0 NS(\alpha) \ln(T/\gamma) T^{1-\frac{a}{4}} )$, where $S(\alpha):=\sum_{i=1}^N\frac{1}{\alpha_i^2}$; see Appendix C.2 for more details.
{Note that the poly-logarithmic term $\ln(T)$ in the regret bound achieved by Algorithm \ref{alg:algorithm3} is dominated by the polynomial term $T^{1-\frac{a}{4}}$ when $T$ is large. Recall also that Algorithm \ref{alg:algorithm1} achieves regret ${\rm{R}}_{C_i}^1(T) = \mathcal{O}( \sqrt{D_x U d_i L_0}N^{\frac{1}{4}}\alpha_i^{-\frac{1}{2}} \sqrt{\ln(T/\gamma)} T^{1-\frac{a}{4}} )$. Ignoring the logarithmic dependence in the regret for large $T$, we obtain that the regret bound achieved by Algorithm \ref{alg:algorithm3} is strictly less than that achieved by Algorithm \ref{alg:algorithm1} when $U>D_x d_i L_0N^{\frac{3}{2}} \alpha_i S^2(\alpha)$. }
\section{Numerical Experiments}\label{sec_simu}
\begin{figure}[t] 
\begin{center}
\centerline{\includegraphics[width=0.8\columnwidth]{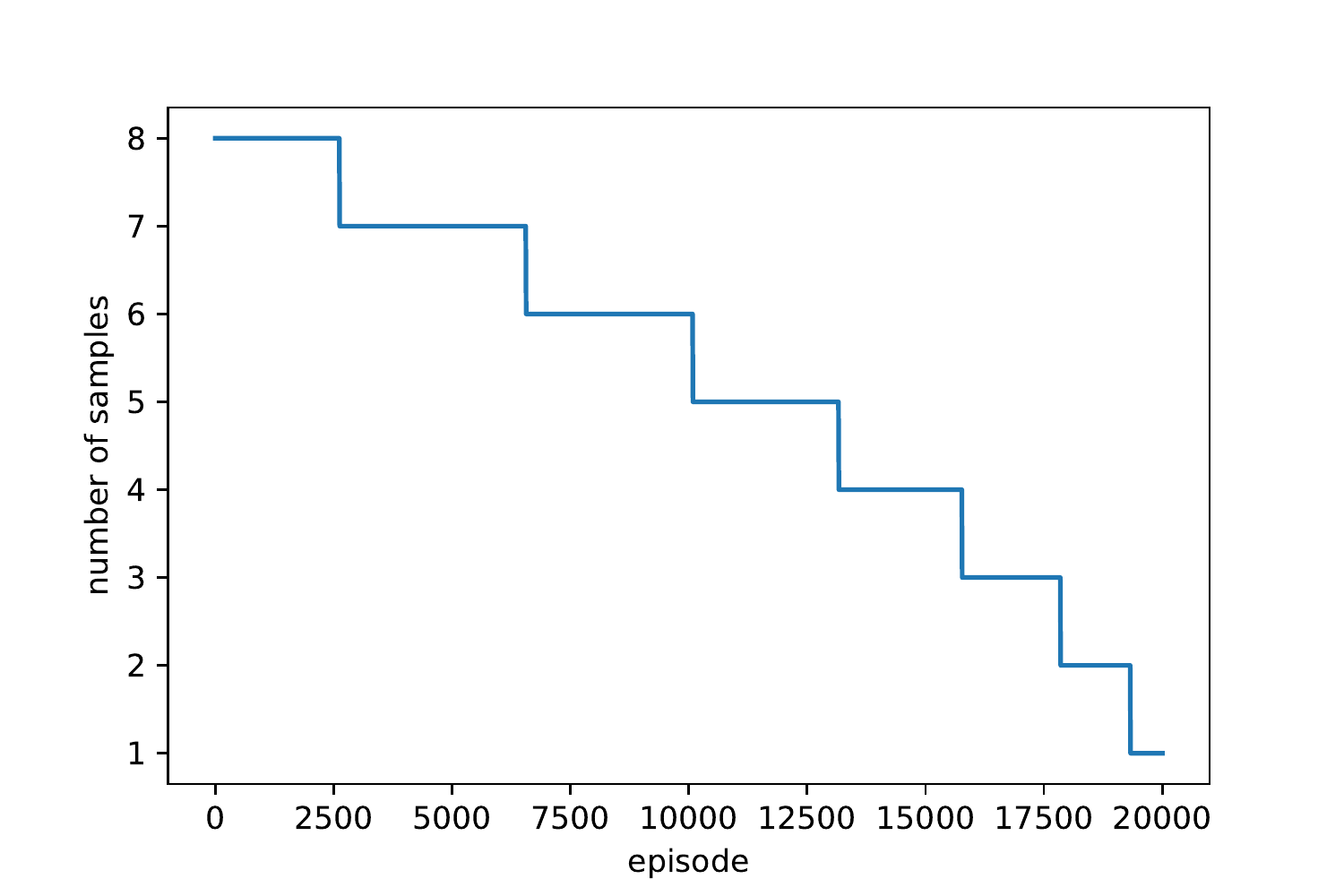}}
\caption{The number of samples of Algorithm \ref{alg:algorithm1}, \ref{alg:algorithm2} and \ref{alg:algorithm3}.}
\label{fig_samples}
\end{center}
\vskip -0.4in
\end{figure}
\begin{figure}[t]
\begin{center}
\centerline{\includegraphics[width=0.8\columnwidth]{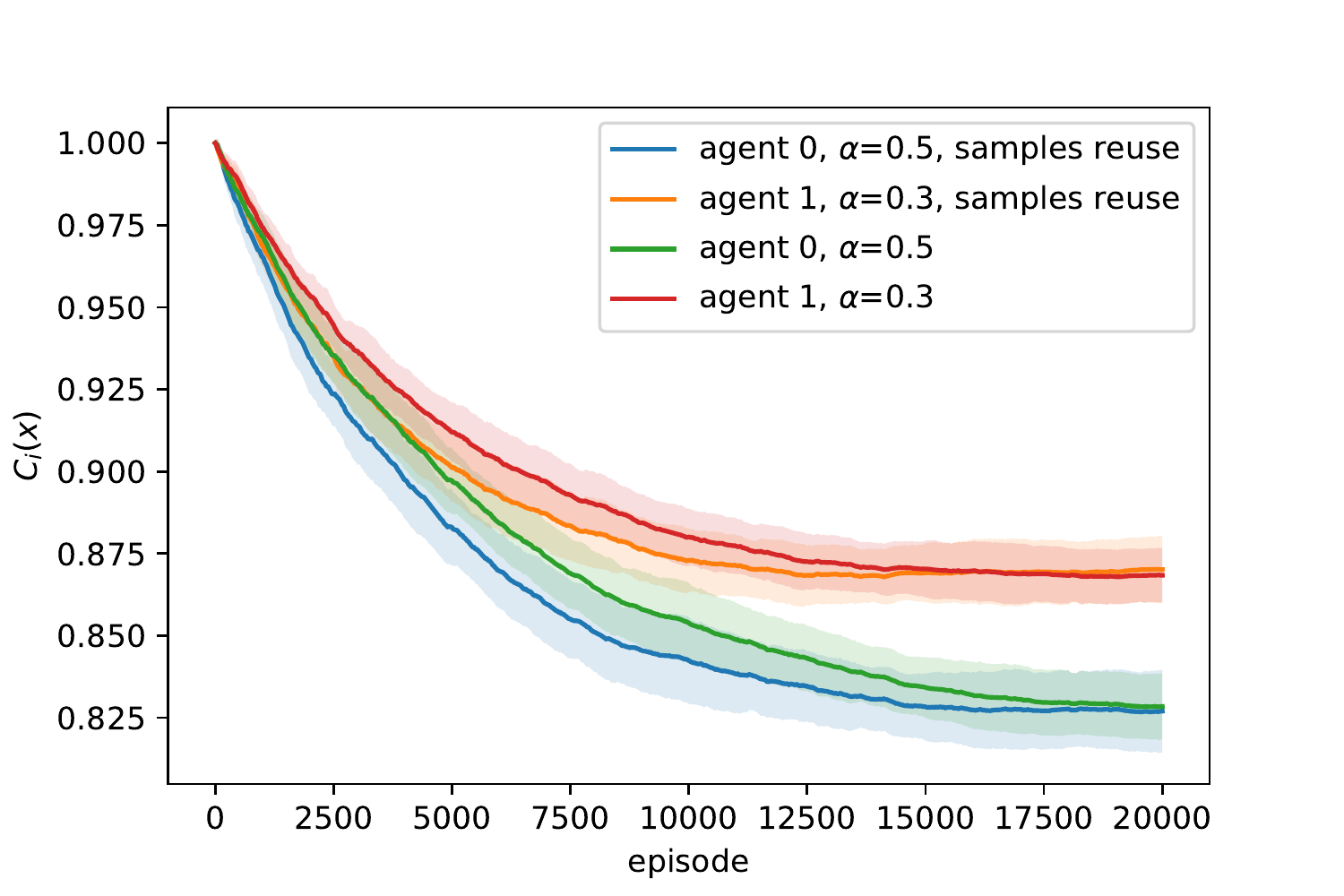}}
\caption{CVaR values achieved by Algorithm \ref{alg:algorithm1} (green and red) and Algorithm \ref{alg:algorithm2} (blue and orange). The solid lines and shades are averages and standard deviations over 20 runs.}
\label{fig_reuse samples}
\end{center}
\vskip -0.4in
\end{figure}
\begin{figure}[t]
\begin{center}
\centerline{\includegraphics[width=0.8\columnwidth]{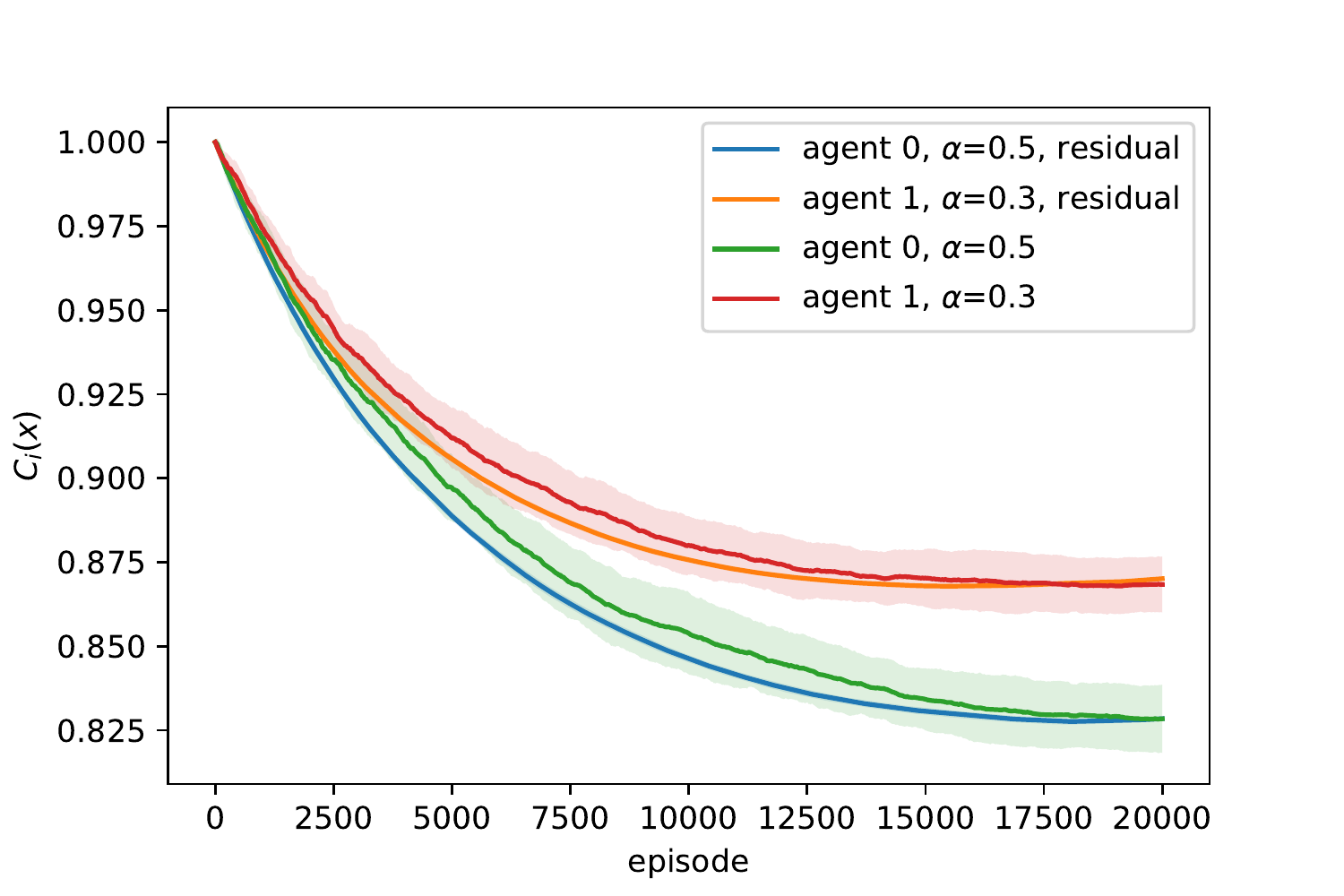}}
\caption{CVaR values achieved by Algorithm \ref{alg:algorithm1} (green and red) and Algorithm \ref{alg:algorithm3} (blue and orange). The solid lines and shades are averages and standard deviations over 20 runs.}
\label{fig_residual}
\end{center}
\vskip -0.4in
\end{figure}
Consider a Cournot game involving two risk-averse firms (agents) in the same market with different risk levels $\alpha_i$.
We let $\alpha_0 =0.5$ and $\alpha_1=0.3$, i.e, firm $1$ is more risk sensitive than firm $0$.
Suppose that firm $i$ supplies the market with a quantity $x_i$, and the total supply of both firms defines the price of the goods in the market.
We assume that the cost function for each firm $i\in\{0,1\}$ is defined as $J_i=-(2-\sum_jx_j)x_i+0.1x_i+\xi_i x_i+1$, where $\xi_i\sim U(0,1)$ is a uniform random variable. The cost term $\xi_i x_i$ models the uncertainty in the market, which is proportional to production.
The goal of each firm is to minimize the CVaR of their local cost function.

Recall the definition of the regret in Section \ref{sec_prob}, and note that it is not possible to compare the performance of Algorithm \ref{alg:algorithm1}, \ref{alg:algorithm2} and \ref{alg:algorithm3} in terms of regret, since the baseline term $\mathop{\rm{min}}_{\tilde{x}_i \in \mathcal{X}_i} \sum_{t=1}^{T} C_i(\tilde{x}_i,x_{-i,t})$ depends on the given sequence of $\{x_{-i,t}\}_{t=1}^T$ and there is no golden rule to help select this sequence. In addition, computing the analytical solution to the baseline term is challenging due to the variational definition of CVaR. Instead, in what follows, we compare Algorithms 1,2 and 3 in terms of their empirical performances.
Specifically, we run Algorithms 1, 2, and 3 and calculate the CVaR values achieved for each action. 
The algorithm with lower CVaR values is preferred since it achieves better performance in terms of risk-aversion. 

The number of samples used by Algorithm \ref{alg:algorithm1} is determined by equation \eqref{eq:sample_strategy} and is shown in Figure \ref{fig_samples}. 
We implement a hybrid sampling strategy for Algorithm \ref{alg:algorithm2} and select the switching time step as 15000, after which prior samples are reused as in equation \eqref{eq:edf2_alg2}.
The reason for this choice is that after 15000 time steps the number of samples becomes small (less than or equal to 3), which causes large errors in CVaR estimation. 
All other parameters in Algorithms \ref{alg:algorithm1}, \ref{alg:algorithm2} and \ref{alg:algorithm3} are tuned so that the three algorithms achieve individually their best performance. Figure \ref{fig_reuse samples} compares empirically the performance of Algorithms 1 and 2. 
We observe that both Algorithms 1 and 2 both converge to the same CVaR values, but Algorithm \ref{alg:algorithm2} that reuses samples converges at a faster speed.
Indeed, the learning rates of both algorithms depend on the number of samples; Algorithm \ref{alg:algorithm2} converges faster because sample reuse increases the effective number of samples per iteration and, as a result, decreases the CVaR estimation errors. This allows for a larger learning rate. 
Figure \ref{fig_residual} shows the variance reduction effect achieved by Algorithm \ref{alg:algorithm3} that employs residual feedback to estimate the CVaR gradients. Note the very low variance (almost non-existent) associated with the blue and orange curves. As with Algorithm \ref{alg:algorithm2}, residual feedback allows Algorithm \ref{alg:algorithm3} to use a larger learning rate and still converge to the same CVaR values as Algorithm \ref{alg:algorithm1}. Additional numerical simulations for different sampling strategies are provided in Appendix D, where we also discuss convergence to the Nash equilibrium in practice.

Motivated by the improvements in performance that can be achieved by reusing prior samples to estimate the CVaR values (Algorithm \ref{alg:algorithm2}) and relying on residual feedback to estimate the CVaR gradients (Algorithm \ref{alg:algorithm3}), it is of interest to analyze the combined effect of these methods for risk-averse learning in online convex games. 
However, the theoretical analysis of this method is nontrivial and, for this reason, it is left for future research.

\section{Conclusion}\label{sec_conclu}
In this work, we proposed a first no-regret algorithm for risk-averse online convex games. Our algorithm relied on a new sampling strategy to estimate the CVaR values of the agents' cost functions, and a zeroth-order estimator of the CVaR gradients to update the agents' actions. To further improve the regret bounds achieved by our algorithm, we proposed two novel modifications; one that reuses samples from the previous iteration to better estimate the CVaR values and another that uses residual feedback to reduce the variance of the CVaR gradient estimation.  We illustrated our proposed method on an online market example modeled as a Cournot game.

\section*{Acknowledgements}
This work is supported in part by AFOSR under award \#FA9550-19-1-0169 and by NSF under award CNS-1932011.

\bibliography{ref}
\bibliographystyle{icml2022}
\newpage
\onecolumn
\section*{A. Proofs of Key Results Supporting Algorithm \ref{alg:algorithm1}}

\subsection*{A.1 Proof of Lemma \ref{lemma:cdelta_property}}\label{proof_lemma:cdelta_property}
\begin{proof}
1. 
From the convexity of $C_i(x_i,x_{-i})$, we can get that $C_i(\theta p_1+ (1-\theta) p_2,x_{-i})\leq \theta C_i(p_1,x_{-i})+(1-\theta)C_i(p_2,x_{-i}) $ for any $p_1,p_2 \in \mathcal{X}_i^{\delta}$ and $\theta \in [0,1]$, . Here we denote $\mathop{\mathbb{E}}_{w_i\sim \mathbb{B}_i, u_{-i}\sim \mathbb{S}_{-i}}$ as $\mathbb{E}$ due to space limit. Thus we have
\begin{align*}
&C_i^{\delta}(\theta p_1+(1-\theta)p_2,x_{-i})\nonumber \\
=& \mathbb{E} \big[ C_i(\theta p_1+ (1-\theta) p_2+\delta w_i,x_{-i}+\delta u_{-i}) \big] \nonumber \\
=& \mathbb{E} \big[ C_i(\theta (p_1+\delta w_i) +(1-\theta) (p_2+\delta w_i),x_{-i}+\delta u_{-i}) \big] \nonumber \\
\leq &  \mathbb{E} \big[ \theta  C_i( p_1+\delta w_i,x_{-i}+\delta u_{-i})+ (1-\theta)C_i( p_2+\delta w_i,x_{-i}+\delta u_{-i})  \big] \nonumber \\
=&  \theta  C_i^{\delta}( p_1+\delta w_i,x_{-i}) + (1-\theta)C_i^{\delta}( p_2+\delta w_i,x_{-i}), \nonumber \\
\end{align*}
which completes the proof. \\
2. According to the definition of $C_i^{\delta}$ function, we have $|C_i^{\delta}(x)-C_i^{\delta}(y)|=|\mathop{\mathbb{E}}_{w_i\sim \mathbb{B}_i, u_{-i}\sim \mathbb{S}_{-i}}[C_i(x_i+\delta w_i,x_{-i}+\delta u_{-i}) - C_i(y_i+\delta w_i,y_{-i}+\delta u_{-i}) ]| \leq \mathop{\mathbb{E}}_{w_i\sim \mathbb{B}_i, u_{-i}\sim \mathbb{S}_{-i}}[ L_0 \left\| x-y\right\|]\leq L_0 \left\| x-y\right\| $. The proof is complete. \\
3. Since the $C_i$ function is $L_0$-Lipschitz continuous, we have that 
\begin{align*}
    |C_i^{\delta}(x)-C_i(x)| 
    =& |\mathbb{E} [C_i(x_i+\delta w_i,x_{-i}+\delta u_{-i})] -C_i(x)| \nonumber \\
    \leq & L_0 \left\| (\delta w_i, \delta u_{-i})\right\| \\
    \leq  L_0 \delta \sqrt{N},\nonumber
\end{align*}
where the last inequality is due to the fact that $\left\| w_i\right\|\leq 1$, $\left\| u_i\right\|\leq 1$. 
\end{proof}

\subsection*{A.2 Proof of Lemma \ref{lemma:CVaR_bound}}\label{proof_lemma:CVaR_bound}
Before the proof of Lemma \ref{lemma:CVaR_bound}, we first give the following result.
\begin{lemma}\label{lemma_Expectation}
Let $F$ be the CDF of a non-negative random variable bounded by $U$, then we have that
\begin{align*}
    \mathbb{E}_F[X-\nu]_{+}=\int_{0}^U (1-F(y))dy-\nu + \int_{0}^{\nu} F(y) dy.
\end{align*}
\end{lemma}
\begin{proof}
It follows that
\begin{align*}
    \mathbb{E}_F[X-\nu]_{+}
    &= \mathbb{E}_F[(X-\nu) \mathbf{1}\{X>\nu \} ] \nonumber \\
    &=\mathbb{E}_F[(X-\nu) (1-\mathbf{1}\{X\leq \nu \}) ]\nonumber \\
    &=\mathbb{E}_F[X]-\nu + \mathbb{E}_F[X \mathbf{1}\{X\leq \nu \}] +\nu F(\nu) .
\end{align*}
Since $a=\int_{0}^a dy= \int_{0}^{\infty} \mathbf{1}\{ y\leq a\}dy$, we obtain that
\begin{align*}
     \mathbb{E}_F[X-\nu]_{+} 
    &=\mathbb{E}_F[X]-\nu  +\nu F(\nu) - \mathbb{E}_F\left[  \mathbf{1}\{X\leq \nu\}  \int_{0}^{\infty} \mathbf{1}\{ y\leq X \}dy \right] \nonumber \\
    &= \mathbb{E}_F[X]-\nu  +\nu F(\nu) - \int_{0}^{\infty} \mathbb{P}_F(y\leq X\leq \nu)dy \nonumber \\
    &= \mathbb{E}_F[X]-\nu  +\nu F(\nu) -\int_{0}^{\nu} (F(\nu)-F(y))dy \nonumber \\
    &=\int_{0}^U (1-F(y))dy-\nu +\int_{0}^{\nu} F(y) dy.
\end{align*}
This ends the proof.  
\end{proof}

Now we are ready to give the proof of Lemma \ref{lemma:CVaR_bound}.
\begin{proof}
According to the CVaR property, we have ${\rm{CVaR}}_{\alpha}[F]=\nu_F+\frac{1}{\alpha}\mathbb{E}_F[X-\nu_F]_{+}$, ${\rm{CVaR}}_{\alpha}[G]=\nu_G+\frac{1}{\alpha}\mathbb{E}_G[X-\nu_G]_{+}$. Since $\nu_F$ is the value that minimizes the ${\rm{CVaR}}_{\alpha}[F]$, we further obtain that 
\begin{align}
    &{\rm{CVaR}}_{\alpha}[F]-{\rm{CVaR}}_{\alpha}[G]  \leq \frac{1}{\alpha}\left( \mathbb{E}_F[X-\nu_G]_{+}-\mathbb{E}_G[X-\nu_G]_{+} \right).
\end{align}

Using Lemma \ref{lemma_Expectation},  we can obtain
\begin{align*}
    &{\rm{CVaR}}_{\alpha}[F]-{\rm{CVaR}}_{\alpha}[G] \nonumber \\
    \leq& \frac{1}{\alpha}\Big(  \int_0^U (1-F(y)) dy-\nu_G + \int_{0}^{\nu_G} F(y)dy  -\int_0^U (1-G(y)) dy+\nu_G - \int_{0}^{\nu_G} G(y)dy\Big) \nonumber \\
    \leq & \frac{1}{\alpha}\left(  \int_0^U (1-F(y)) dy-\int_0^U (1-G(y)) dy  + \int_{0}^{\nu_G}(F(y)-G(y))dy  \right) \nonumber \\
    \leq & \frac{1}{\alpha} \left(\int_0^U (G(y)-F(y)) dy + \int_{0}^{\nu_G}(F(y)-G(y))dy  \right) \nonumber \\
    \leq &  \frac{1}{\alpha} \int_{\nu_G}^U (G(y)-F(y)) dy \nonumber \\
    \leq &  \frac{U}{\alpha}
     \mathop{\rm{sup}}_{y} |F(y)-G(y)|. \nonumber \\
\end{align*}
By symmetry, we can bound ${\rm{CVaR}}_{\alpha}[G]-{\rm{CVaR}}_{\alpha}[F]$ as well and the proof is omitted. The proof is complete.
\end{proof}

\subsection*{A.3 Proof of Lemma \ref{lemma:sum_epsihat_bound}}\label{proof_lemma:sum_epsihat_bound}
We first give the DKW inequality, which is helpful in our subsequent analysis.
\begin{lemma}[DKW inequality] Let $F$ be the CDF of a random variable and $\hat{F}$ be the empirical CDF obtained by $n$ i.i.d. samples. For a gien constant $\epsilon>0$, we have 
\begin{align*}
    \mathbb{P}\left\{\sup_y|F(y)-\hat{F}(y)|> \epsilon \right\} \leq 2e^{-2n\epsilon^2}.
\end{align*}
\end{lemma}
Now we are ready to show the proof of Lemma \ref{lemma:sum_epsihat_bound}.
\begin{proof}
According to the DKW inequality, we have that
\begin{align}\label{event_At}
    \mathbb{P}\left\{ \mathop{\rm{sup}}_{y} |F_{i,t}(y)-\hat{F}_{i,t}(y)| \geq \sqrt{\frac{ {\rm{ln}}(2/\bar{\gamma}) }{2n_t}}\right\} \leq \bar{\gamma}.
\end{align}
Define the events in (\ref{event_At}) as $A_t$. Recall that $\gamma=\bar{\gamma}T$, then the following holds
\begin{align}\label{eq_probability1}
\mathop{\rm{sup}}_{y} |F_{i,t}(y)-\hat{F}_{i,t}(y)| \leq \sqrt{\frac{ {\rm{ln}}(2T/\gamma) }{2n_t}}, \forall t=1,\ldots,T,
\end{align}
with probability at least $1-\gamma$ due to the fact that $1-\mathbb{P} \{ \bigcup_{t=1}^T A_t\} \geq 1- \sum_{t=1}^T\mathbb{P}\{A_t\}\geq 1-T \frac{\gamma}{T}\geq1-\gamma$. 

Combining with the sampling strategy defined in equation (\ref{eq:sample_strategy}), and applying Lemma \ref{lemma:CVaR_bound}, the accumulated error of CVaR estimation can be bounded as follows
\begin{align*}
\sum_{t=1}^T |\hat{\varepsilon}_{i,t}|  &\leq \sum_{t=1}^T \frac{U}{\alpha_i}\sqrt{\frac{ \ln(2T/\gamma)}{2n_t }} 
 \leq  \sum_{t=1}^T \frac{U}{\alpha_i}\sqrt{\frac{ \ln(2T/\gamma)}{2(b  U^2(T-t+1)^a) }} \nonumber \\
& =  \sum_{t=1}^T  \frac{1}{\alpha_i}\sqrt{\frac{ \ln(2T/\gamma)}{2b  t^a }}
\leq  \frac{1}{\alpha_i}\sqrt{\frac{ \ln(2T/\gamma)}{2b   }} \left( 1+\int_{t=1}^T \frac{1}{\sqrt{t}} dt\right)\nonumber \\
& \leq  \frac{1}{\alpha_i}\sqrt{\frac{ \ln(2T/\gamma)}{2b   }} \left( 1+ \frac{1}{1-\frac{a}{2}} (T^{1-\frac{a}{2}}-1)  \right)\nonumber 
\leq \frac{1}{\alpha_i}\sqrt{\frac{2\ln(2T/\gamma)}{ b }}T^{1-\frac{a}{2}},
\end{align*}
where the last inequality holds due to the fact that $\frac{1}{1-\frac{a}{2}}<2$, which completes the proof.
\end{proof}

\subsection*{A.4 Proof of Lemma \ref{lemma:regret_decomp}}\label{proof_lemma:regret_decomp}
\begin{proof} We first present an observation.\\
\textbf{Observation 1}: There exist a constant $\Omega>0$ such that
\begin{align*}
&\quad  \mathop{{\rm{min}}}_{\tilde{x}_i \in \mathcal{X}_i^{\delta}} \sum_{t=1}^{T}  C_i^{\delta}(\tilde{x}_i, x_{-i,t}) \leq \mathop{{\rm{min}}}_{\tilde{x}_i \in \mathcal{X}_i} \sum_{t=1}^{T} C_i^{\delta}(\tilde{x}_i, x_{-i,t})+\Omega L_{0} \delta T.
\end{align*}
\begin{proof}
Let $x_i^{1}= \mathop{{\rm{ arg \; min}}}_{\tilde{x}_i \in \mathcal{X}_i^{\delta}} \sum_{t=1}^{T} C_i^{\delta}(\tilde{x}_i, x_{-i,t})$,
$x_i^{2}= \mathop{{\rm{ arg \; min}}}_{\tilde{x}_i \in \mathcal{X}_i} \sum_{t=1}^{T} C_i^{\delta}(\tilde{x}_i, x_{-i,t})$. According to Lemma 3 in \cite{tatarenko2020bandit}, $\left\|x_i^{1} -x_i^{2}\right\| = O(\delta)$, i.e, there exist a constant $\Omega>0$ such that $\left\|x_i^{1} -x_i^{2}\right\|\leq \Omega \delta$. Adding that $C_i^{\delta}(x_i,x_{-i})$ is $L_{0}$-Lipschitz continuous, we can easily obtain the claim. 
\end{proof}

Let $x_{\delta_i}^{*}= \mathop{\rm{min}}_{\tilde{x}_i \in \mathcal{X}_i} \sum_{t=1}^{T} C_i(\tilde{x}_i,x_{-i,t}) $. We have that
\begin{align}\label{append:eq:diff}
 \left\| x_{i,t+1}-x_{\delta_i}^{*}\right\|^2 
&= \left\| \mathcal{P}_{\mathcal{X}_i^{\delta}}(x_{i,t}- \eta \hat{g}_{i,t} )-x_{\delta_i}^{*}\right\|^2  \leq \left\|x_{i,t}- \eta \hat{g}_{i,t} -x_{\delta_i}^{*}\right\|^2  \nonumber \\
& = \left\| x_{i,t}-x_{\delta_i}^{*}\right\|^2 +  \eta^2 \left\|\hat{g}_{i,t}\right\|^2 - 2 \eta \langle \hat{g}_{i,t}, x_{i,t}-x_{\delta_i}^{*} \rangle.
\end{align}

Since $C_i^{\delta}(x_i,x_{-i})$ is convex in $x_i$, we have that
\begin{align}\label{append:eq:ineq}
& C_i^{\delta}(x_t) - C_i^{\delta}(x_{\delta_i}^{*},x_{-i,t}) 
\leq \nabla_i C_i^{\delta}(x_t) (x_{i,t}-x_{\delta_i}^{*}) = \mathbb{E} \langle \hat{g}_{i,t}, x_{i,t} - x_{\delta_i}^{*} \rangle - \mathbb{E} \langle \frac{d_i}{\delta}\hat{\varepsilon}_{i,t} u_{i,t}, x_{i,t} - x_{\delta_i}^{*} \rangle  \nonumber \\
& \leq \frac{1}{2\eta} \mathbb{E} \big(\left\| x_{i,t}-x_{\delta_i}^{*}\right\|^2 - \left\|x_{i,t+1}-x_{\delta_i}^{*}\right\|^2   \big) + \frac{\eta}{2} \mathbb{E} \left\|\hat{g}_{i,t}\right\|^2  + \mathbb{E} \big[ \frac{d_i}{\delta}\left\| \hat{\varepsilon}_{i,t}\right\| \left\| x_{i,t}-x_{\delta_i}^{*}\right\| \big],
\end{align}
where the last inequality follows from equation \eqref{append:eq:diff}.

Taking the sum from $t=1$ to $T$ on both sides of equation \eqref{append:eq:ineq}, we obtain that
\begin{align}\label{eq:Ci_delta}
&   \sum_{t=1}^{T} C_i^{\delta}(x_{i,t},x_{-i,t}) -  \mathop{{\rm{min}}}_{\tilde{x}_i \in \mathcal{X}_i^{\delta}} \sum_{t=1}^{T} C_i^{\delta}(\tilde{x}_i, x_{-i,t})  \nonumber \\
\leq & \frac{\left\| x_{i,1}- x_{\delta_i}^{*}\right\|^2}{2\eta} + \frac{\eta}{2}  \mathbb{E} \big[ \sum_{t=1}^{T} \left\|\hat{g}_{i,t}\right\|^2 \big] +  \mathbb{E}\big[ \sum_{t=1}^{T} \frac{d_i}{\delta}\left\| \hat{\varepsilon}_{i,t}\right\| \left\| x_{i,t}-x_{\delta_i}^{*}\right\| \big] \nonumber \\
\leq &  \frac{D_x^2}{2\eta} + \frac{\eta}{2}  \mathbb{E} \big[ \sum_{t=1}^{T} \left\|\hat{g}_{i,t}\right\|^2 \big] + \frac{d_i D_x }{\delta}\mathbb{E} \big[ \sum_{t=1}^{T} \left\| \hat{\varepsilon}_{i,t}\right\| \big].
\end{align}
Recalling that $|C_i^{\delta}(x)-C_i(x)|\leq L_0\sqrt{N}\delta$, and $C_i^{\delta}(x)$ is Lipschitz continuous, it follows that
\begin{align*}
{\rm{R}}_i^C(T) 
&= \sum_{t=1}^{T} C_i( \hat{x}_t) - \mathop{\rm{min}}_{\tilde{x}_i \in \mathcal{X}_i} \sum_{t=1}^{T} C_i(\tilde{x}_i,\hat{x}_{-i,t})\nonumber \\
& \leq \sum_{t=1}^{T} C_i^{\delta}(\hat{x}_t) - \mathop{\rm{min}}_{\tilde{x}_i \in \mathcal{X}_i} \sum_{t=1}^{T} C_i^{\delta}(\tilde{x}_i,\hat{x}_{-i,t}) +2 \delta L_0\sqrt{N}T \nonumber \\
& \leq \sum_{t=1}^{T} C_i^{\delta}(x_t) - \mathop{\rm{min}}_{\tilde{x}_i \in \mathcal{X}_i} \sum_{t=1}^{T} C_i^{\delta}(\tilde{x}_i,x_{-i,t}) +4 \delta L_0\sqrt{N}T.
\end{align*}
Applying Observation 1 and (\ref{eq:Ci_delta}) into the inequality above, it gives that
\begin{align}
{\rm{R}}_i^C(T) 
& \leq \sum_{t=1}^{T} C_i^{\delta}(x_t) - \mathop{\rm{min}}_{\tilde{x}_i \in \mathcal{X}_i} \sum_{t=1}^{T} C_i^{\delta}(\tilde{x}_i,x_{-i,t}) +4 \delta L_0\sqrt{N}T \nonumber \\
& \leq \sum_{t=1}^{T} C_i^{\delta}(x_t) - \mathop{\rm{min}}_{\tilde{x}_i \in \mathcal{X}_i^{\delta}} \sum_{t=1}^{T} C_i^{\delta}(\tilde{x}_i,x_{-i,t})   + \Omega L_0 \delta T  +4 \delta L_0\sqrt{N}T \nonumber \\
& \leq  \frac{D_x^2}{2\eta}   + \frac{\eta}{2}  \mathbb{E} \big[ \sum_{t=1}^{T} \left\|\hat{g}_{i,t}\right\|^2 \big]+ \frac{d_i D_x }{\delta}\mathbb{E} \big[ \sum_{t=1}^{T} \left\| \hat{\varepsilon}_{i,t}\right\| \big] +4 \delta L_0\sqrt{N}T  + \Omega L_0 \delta T \nonumber \\
& \leq \frac{D_x^2}{2\eta} + \frac{d_i^2 U^2 \eta}{2\delta^2}  T  + \frac{d_i D_x }{\delta}\mathbb{E} \big[ \sum_{t=1}^{T} \left\| \hat{\varepsilon}_{i,t}\right\| \big]  + (4\sqrt{N}+\Omega)L_0 \delta T,
\end{align}
where the last inequality follows from the fact that $|\hat{C}_i|\leq U$ and $\left\|\hat{g}_{i,t}\right\|=\left\|\frac{d_i}{\delta} {\rm{CVaR}}_{\alpha_i} [\hat{F}_{i,t} ] u_{i,t}\right\| \leq \frac{d_iU}{\delta}$. 
\end{proof}

\section*{B. Proofs of Key Results Supporting Algorithm \ref{alg:algorithm2}}\label{appendix_B}
\begin{algorithm}[t]
\caption{Risk-averse learning with sample reuse} \label{alg:algorithm2}
\begin{algorithmic}[1]
    \REQUIRE Initial value $x_0$, step size $\eta$, parameters $a$, $b$, $\delta$, $T$, risk level $\alpha_i$, $i=1,\cdots,N$.
    \FOR{$episode \; t=1,\ldots,T$}
        \STATE Select $n_t=\lceil b U^2 (T-t+1)^a\rceil$
        \STATE Each agent samples $u_{i,t} \in \mathbb{S}^{d_i}$,  $i=1,\ldots, N$
        \STATE Each agent play $\hat{x}_{i,t}=x_{i,t}+\delta u_{i,t} $, $i=1,\ldots, N$
        \FOR{$j=1,\ldots,n_t$}
        \STATE Let all agents play $\hat{x}_{i,t}$  
        \STATE Obtain $J_i(\hat{x}_{i,t},\hat{x}_{-i,t},\xi_{i}^j)$
        \ENDFOR
        \FOR{agent $ i=1,\ldots,N$}
        \STATE Build EDF $\tilde{F}_{i,t}(y)$ 
        \STATE Calculate CVaR estimate: $ {\rm{CVaR}}_{\alpha_i}[\tilde{F}_{i,t}(y)] $ 
        \STATE Construct gradient estimate $\tilde{g}_{i,t}=\frac{d_i}{\delta} {\rm{CVaR}}_{\alpha_i} [\tilde{F}_{i,t}(y) ] u_{i,t}$
        \STATE Update $x$: $x_{i,t+1} \leftarrow \mathcal{P}_{\mathcal{X}_i^{\delta}} ( x_{i,t} - \eta \tilde{g}_{i,t})$
        \ENDFOR
    \ENDFOR
\end{algorithmic}
\end{algorithm}
\subsection*{B.1 Proof of Lemma \ref{lemma:Ftilde_bound}}
\begin{proof}
In order to give the formal proof, we need to introduce some new definitions.
Define a new random variable $\check{J}_{i,t}=z_t J_{i,t}+(1-z_t) J_{i,t-1}$, where $J_{i,t}$, $J_{i,t-1}$ are abbreviation of $J_i(\hat{x}_t,\xi_i)$, $J_i(\hat{x}_{t-1},\xi_i)$, $z_t$ is an independent Bernoulli random variable with $\mathbb{P}\{ z_t=1\}=\frac{n_t}{N_t}$. Define the events $B_t:=\{ \# (z_t=1)=n_t \; {\rm{from}} \; N_t\; {\rm{samples}} \; {\rm{of}}\;  \check{J}_{i,t}\}$, where $\#$ denotes the times that the event occurs. 
Then we define
the conditional expectation of $\check{J}_{i,t}$ given event $B_t$ as $\bar{J}_{i,t}=\mathbb{E}[\check{J}_{i,t}|B_t]$, which is still a random variable \cite{durrett2019probability}.
Define the CDF of $\bar{J}_{i,t}$ as $\bar{F}_{i,t}$, then we have that
\begin{align}
    &\bar{F}_{i,t}(y) =\mathbb{P}\{ \bar{J}_{i,t}\leq y \} = \mathbb{P}\{z_t J_{i,t}+(1-z_t) J_{i,t-1}\leq y | B_t \}  \nonumber \\
    &= \mathbb{P}\{z_t=1 |B_t\} \mathbb{P}\{J_{i,t} \leq y\} + \mathbb{P}\{z_t=0 |B_t\} \mathbb{P}\{J_{i,t-1} \leq y\} =\frac{n_t}{N_t}F_{i,t} + \frac{n_{t-1}}{N_t}F_{i,t-1},
\end{align}
where the last equality is obtained by definitions of the random variables $z_t$, $J_{i,t}$, $J_{i,t-1}$. The second last equality is due to the fact that $B_t$ is only related to $z_t$, and $z_t$ is independent of $J_{i,t}$,$J_{i,t-1}$.  
Then it gives that
\begin{align}
    &\mathop{\rm{sup}}_{y} |\bar{F}_{i,t}(y)- F_{i,t}(y)| \nonumber \\
    \leq &\mathop{\rm{sup}}_{y}  |\frac{n_{t-1}}{N_t}(F_{i,t-1}(y)- F_{i,t}(y) )| \nonumber \\
    \leq& \frac{1}{2} \mathop{\rm{sup}}_{y} |F_{i,t-1}(y)- F_{i,t}(y) |,
\end{align}
where the last inequality is due to $n_{t-1}\leq n_t$ and thus $\frac{n_{t-1}}{N_t}\leq \frac{1}{2}$.
From the definition of $\bar{J}_{i,t}$, we have that $\tilde{F}_{i,t}$ is an EDF of $\bar{F}_{i,t}$. Applying DKW inequality, we have that, for $t\geq 2$
\begin{align}\label{event_Et}
    \mathbb{P}\left\{ \mathop{\rm{sup}}_{y} |\tilde{F}_{i,t}(y)-\bar{F}_{i,t}(y)| \geq \sqrt{\frac{ {\rm{ln}}(2/\bar{\gamma}) }{2(n_t+n_{t-1})}}\right\} \leq \bar{\gamma}.
\end{align}
Define the event in (\ref{event_Et}) as $E_t$. 
Then the following holds
\begin{align}\label{eq_probability2}
\mathop{\rm{sup}}_{y} |\tilde{F}_{i,t}(y)-\bar{F}_{i,t}(y)| \leq \sqrt{\frac{ {\rm{ln}}(2T/\gamma) }{2(n_t+n_{t-1})}}, \forall t=2,\ldots,T,
\end{align}
with the probability at least $1-\gamma$ due to the fact that $1-\mathbb{P} \{ \bigcup_{t=2}^T E_t\} \geq 1- \sum_{t=2}^T\mathbb{P}\{E_t\}\geq 1-(T-1) \frac{\gamma}{T}\geq 1-\gamma$.
Together with Assumption \ref{assump_F_lips}, the following holds 
\begin{align}
    \mathop{\rm{sup}}_{y} |\tilde{F}_{i,t}(y)-F_{i,t}(y)| =& \mathop{\rm{sup}}_{y} |\tilde{F}_{i,t}(y)- \bar{F}_{i,t}(y) +\bar{F}_{i,t}(y) -F_{i,t}(y)| \nonumber \\
    \leq & \mathop{\rm{sup}}_{y} |\tilde{F}_{i,t}(y)- \bar{F}_{i,t}(y)| + \mathop{\rm{sup}}_{y} |\bar{F}_{i,t}(y)- F_{i,t}(y)| \nonumber \\
    \leq & \sqrt{\frac{ {\rm{ln}}(2/\bar{\gamma})} {2(n_t+n_{t-1})}}
    + \frac{1}{2} \mathop{\rm{sup}}_{y} |F_{i,t-1}(y)- F_{i,t}(y)| \nonumber \\
     \leq & \sqrt{\frac{ {\rm{ln}}(2/\bar{\gamma})} {2(n_t+n_{t-1})}} + \frac{C_1 \delta +C_2}{2} \left\| x_t-x_{t-1}\right\|,
\end{align}
with probability at least $1-\gamma$ for $\forall t=2,\ldots,T$. Then, applying Lemma \ref{lemma:CVaR_bound}, it gives that
\begin{align}
    |\tilde{\varepsilon}_{i,t}|& \leq \frac{U}{\alpha_i} \mathop{\rm{sup}}_{y} |\tilde{F}_{i,t}(y)-F_{i,t}(y)| \nonumber \\ 
    &\leq \frac{U}{\alpha_i} \left( \sqrt{\frac{ {\rm{ln}}(2/\bar{\gamma})} {2(n_t+n_{t-1})}} + \frac{C_1 \delta +C_2}{2} \left\| x_t-x_{t-1}\right\|\right) \nonumber \\
    & \leq \frac{U}{\alpha_i}\left( \sqrt{\frac{ \ln(2T/\gamma)}{2(n_t+n_{t-1}) }} + \frac{(C_1\delta+C_2) d_i U\sqrt{N}\eta}{2\delta} \right),
\end{align}
where the inequality holds due to the fact that  $\left\| \tilde{g}_{i,t} \right\|=\left\|\frac{d_i}{\delta} {\rm{CVaR}}_{\alpha_i} [\tilde{F}_{i,t} ] u_{i,t}\right\| \leq \frac{d_i}{\delta}U$ and
$\left\|x_t - x_{t-1} \right\|\leq \eta \left\| \tilde{g}_t \right\| \leq \frac{\eta d_i U\sqrt{N}}{\delta}$.
The proof is complete.
\end{proof}

\subsection*{B.2 Proof of Lemma \ref{lemma:regret_decomp_2}}
\begin{proof}
The proof is very similar to the proof of Lemma \ref{lemma:regret_decomp}. By substituting $\hat{g}_{i,t}$, $\hat{\varepsilon}_{i,t}$ with $\tilde{g}_{i,t}$, $\tilde{\varepsilon}_{i,t}$, we can obtain the claim. The detailed proof is omitted.
\end{proof}

\subsection*{B.3 Proof of Corollary 1}
\begin{proof}
According to the sampling strategy in (\ref{eq:sample_strategy}), we have that
\begin{align*}
    \sum_{t=1}^T|\tilde{\varepsilon}_{i,t}| &=\sum_{t=1}^{t_0}|\tilde{\varepsilon}_{i,t}|+\sum_{t=t_0+1}^T|\tilde{\varepsilon}_{i,t}|  \\
    &\leq \sum_{t=1}^{t_0}|\tilde{\varepsilon}_{i,t}| + \frac{1}{\alpha_i}\sqrt{\frac{2\ln(2T/\gamma)}{ \alpha_i^2 b }}(T-t_0)^{1-\frac{a}{2}}:= \bar{B}_1(t_0).
\end{align*}
When $t\leq t_0$, since $\tilde{F}_{i,t}=\hat{F}_{i,t}$, we have that $\tilde{\varepsilon}_{i,t}= \hat{\varepsilon}_{i,t}$. Next we focus on the term $\sum_{t=t_0+1}^T|\tilde{\varepsilon}_{i,t}|$.
Applying the inequality in (\ref{eq:epsi_tilde_bound}) and substituting $\delta=\frac{\sqrt{D_x U d_i}}{N^{\frac{1}{4}} T^{\frac{a}{4}} \sqrt{\alpha_i L_0} }$,  $\eta=\frac{\sqrt{\alpha_i} D_x^{\frac{3}{2}}}{\sqrt{L_0 U d_i} N^{\frac{1}{4}} T^{\frac{3a}{4}} }$ into it, we have that
\begin{align}
     \sum_{t=t_0+1}^T |\tilde{\varepsilon}_{i,t}| 
    & \leq \frac{U}{\alpha_i} \sum_{t=t_0+1}^T  \sqrt{\frac{ \ln(2T/\gamma)}{4n_{t-1} }}+ \frac{(C_1\delta+C_2) d_i U^2 \sqrt{N} \eta (T-t_0)}{2\alpha_i\delta} \nonumber  \\
    & \leq \sqrt{\frac{ \ln(2T/\gamma)}{\alpha_i^2 b}}(T-t_0)^{1-\frac{a}{2}} + \Sigma_5 (T-t_0)T^{-\frac{3a}{4}}  +\Sigma_6 (T-t_0)T^{-\frac{a}{2}}, \nonumber
\end{align}
where $\Sigma_5=\frac{C_1 D_x^{\frac{3}{2}} N^{\frac{1}{2}}  d_i^{\frac{1}{2}} U^{\frac{3}{2}} }{2 \alpha_i^{\frac{1}{2}} L_0^{\frac{1}{2}}}$, $\Sigma_6=\frac{C_2U N^{\frac{1}{2}}D_x }{2}$.

Set $p= \frac{\ln(T-t_0)}{\ln T} $, then we have $T-t_0=T^p$ with $p\in (0,1)$.
Then, it gives that
\begin{align}
    &\sum_{t=1}^T |\tilde{\varepsilon}_{i,t}| - \bar{B}_1(t_0) \nonumber \\
    &\leq \sqrt{\frac{ \ln(2T/\gamma)}{\alpha_i^2 b}}(T-t_0)^{1-\frac{a}{2}} + \Sigma_5 (T-t_0)T^{-\frac{3a}{4}}  +\Sigma_6 (T-t_0)T^{-\frac{a}{2}} - \sqrt{\frac{ 2 \ln(2T/\gamma)}{\alpha_i^2 b}}(T-t_0)^{1-\frac{a}{2}} \nonumber \\
    &\leq (1-\sqrt{2})\sqrt{\frac{ \ln(2T/\gamma)}{\alpha_i^2 b}} T^{p(1-\frac{a}{2})} + \Sigma_5T ^{p-\frac{3a}{4}} + \Sigma_6 T^{p-\frac{a}{2}} \nonumber \\
    &\leq T^{p-\frac{a}{2}}\left( (1-\sqrt{2})\sqrt{\frac{ \ln(2T/\gamma)}{\alpha_i^2 b}}T^{\frac{a}{2}(1-y)} + \frac{\Sigma_5}{T^{\frac{a}{4}}}  +\Sigma_6 \right),
\end{align}
Define the function $g(T)=(1-\sqrt{2})\sqrt{\frac{ \ln(2T/\gamma)}{\alpha_i^2 b}}T^{\frac{a}{2}(1-y)} + \frac{\Sigma_5}{T^{\frac{a}{4}}}  +\Sigma_6$. 
It can be verified that the function $g(T)$ is monotonically deceasing in $T$ and approaches negative infinity when $T\rightarrow \infty$. 
Then there must exist $T_{\lambda(t_0)}$ such that when $T>T_{\lambda(t_0)}$, we have $g(T)\leq -\lambda$,
and thus $\sum_{t=1}^T |\tilde{\varepsilon}_{i,t}| - \bar{B}_1(t_0)\leq -\lambda T^{p-\frac{a}{2}} $. 
Recall that $ \bar{B}_1(t_0)=\sum_{t=1}^{t_0}|\tilde{\varepsilon}_{i,t}| + \frac{1}{\alpha_i}\sqrt{\frac{2\ln(2T/\gamma)}{ \alpha_i^2 b }}(T-t_0)^{1-\frac{a}{2}}\leq \frac{1}{\alpha_i}\sqrt{\frac{2\ln(2T/\gamma)}{ \alpha_i^2 b }}( T^{1-\frac{a}{2}}-   (T-t_0)^{1-\frac{a}{2}}) + \frac{1}{\alpha_i}\sqrt{\frac{2\ln(2T/\gamma)}{ \alpha_i^2 b }}(T-t_0)^{1-\frac{a}{2}}=B_1$.
By virtue of Lemma \ref{lemma:regret_decomp_2}, it gives that
\begin{align}
    {\rm{R}}_{C_i}^2(T) &\leq  Err(ZO)+\frac{d_i D_x }{\delta}\sum_{t=1}^T |\tilde{\varepsilon}_{i,t}| \nonumber \\
    &\leq  Err(ZO)+Err(CVaR) + \frac{d_i D_x }{\delta}( \sum_{t=1}^T |\tilde{\varepsilon}_{i,t}| - \bar{B}_1(t_0)) \nonumber \\
    &\leq  Err(ZO)+Err(CVaR) - \lambda T^{p-\frac{a}{2}} . \nonumber 
\end{align}
By choosing $\delta=\frac{\sqrt{D_x U d_i}}{N^{\frac{1}{4}} T^{\frac{a}{4}} \sqrt{\alpha_i L_0} }$,  $\eta=\frac{\sqrt{\alpha_i} D_x^{\frac{3}{2}}}{\sqrt{L_0 U d_i} N^{\frac{1}{4}} T^{\frac{3a}{4}} }$, the result follows from the same proof as in Theorem \ref{thm:1}.
The proof is complete.
\end{proof}
\section*{C. Proofs of Key Results Supporting Algorithm \ref{alg:algorithm3}}
\begin{algorithm}[htp]
\caption{Risk-averse learning with residual feedback} \label{alg:algorithm3}
\begin{algorithmic}[1]
    \REQUIRE Initial value $x_0$, step size $\eta$, parameters $a$, $b$, $\delta$, $T$, risk level $\alpha_i$, $i=1,\cdots,N$.
    \FOR{$episode \; t=1,\ldots,T$}
        \STATE Select $n_t=\lceil b U^2 (T-t+1)^a\rceil$
        \STATE Each agent samples $u_{i,t} \in \mathbb{S}^{d_i}$,  $i=1,\ldots, N$
        \STATE Each agent play $\hat{x}_{i,t}=x_{i,t}+\delta u_{i,t} $, $i=1,\ldots, N$
        \FOR{$j=1,\ldots,n_t$}
        \STATE Let all agents play $\hat{x}_{i,t}$  
        \STATE Obtain $J_i(\hat{x}_{i,t},\hat{x}_{-i,t},\xi_{i}^j)$
        \ENDFOR
        \FOR{agent $ i=1,\ldots,N$}
        \STATE Build EDF $\hat{F}_{i,t}(y)$ 
        \STATE Calculate CVaR estimate: $ {\rm{CVaR}}_{\alpha_i}[\hat{F}_{i,t}] $ 
        \STATE Construct gradient estimate $\bar{g}_{i,t}=\frac{d_i}{\delta}\left( {\rm{CVaR}}_{\alpha_i} [\hat{F}_{i,t} ]-{\rm{CVaR}}_{\alpha_i} [\hat{F}_{i,t-1} ]\right) u_{i,t}$
        \STATE Update $x$: $x_{i,t+1} \leftarrow \mathcal{P}_{\mathcal{X}_i^{\delta}} ( x_{i,t} - \eta \bar{g}_{i,t})$
        \ENDFOR
    \ENDFOR
\end{algorithmic}
\end{algorithm}

Observe that the zeroth-order CVaR gradient satisfies the following inequality.  
\begin{lemma}\label{lemma:RF_gradient}
\begin{align}
    \sum_{t=1}^T\left\| \bar{g}_{i,t} \right\|^2 \leq &\frac{1}{1-\beta}\left\|\bar{g}_{1} \right\|^2 + \frac{16d_i^2 N^2L_0^2}{1-\beta}T +\frac{4d_i^2 {\rm{ln}}(2T/\gamma) S(\alpha)}{(1-\beta)b \delta^2}T,
\end{align}
where $\beta=\frac{4d_i^2L_0^2N \eta^2}{\delta^2}$.
\end{lemma}
\subsection*{C.1 Proof of Lemma \ref{lemma:RF_gradient}}
\begin{proof}
According to the definition of $\bar{g}_{i,t}$ in (\ref{eq:grad_alg3}), it gives that
\begin{align}\label{eq_bar_git}
    \left\|\bar{g}_{i,t}\right\|^2 
    &=\frac{d_i^2}{\delta^2} \left( ({\rm{CVaR}}_{\alpha_i} [\hat{F}_{i,t} ]- {\rm{CVaR}}_{\alpha_i} [\hat{F}_{i,t-1} ] ) u_{i,t}\right)^2 \nonumber \\
    &\leq \frac{d_i^2}{\delta^2} \left({\rm{CVaR}}_{\alpha_i} [\hat{F}_{i,t} ]- {\rm{CVaR}}_{\alpha_i} [\hat{F}_{i,t-1} ] \right)^2 \left\| u_{i,t}\right\|^2 \nonumber \\
    &\leq \frac{d_i^2}{\delta^2} \Big( 2 ({\rm{CVaR}}_{\alpha_i} [F_{i,t} ]- {\rm{CVaR}}_{\alpha_i} [F_{i,t-1} ])^2  + 2 (\hat{\varepsilon}_{i,t}-\hat{\varepsilon}_{i,t-1})^2 \Big)  \left\| u_{i,t}\right\|^2,
\end{align}
where the last inequality is due to the fact that $(a+b)^2\leq 2a^2+2b^2$. Note that even though ${\rm{CVaR}}_{\alpha_i} [F_{i,t} ]$ is Lipschitz continuous, the ${\rm{CVaR}}_{\alpha_i} [\hat{F}_{i,t} ]$ is not. We can not bound the difference ${\rm{CVaR}}_{\alpha_i} [\hat{F}_{i,t} ]- {\rm{CVaR}}_{\alpha_i} [\hat{F}_{i,t-1} ]$ directly. Then, by virtue of the Lipschitz property of the function $C_i(x)$, we have that
\begin{align}\label{eq_term1}
    &({\rm{CVaR}}_{\alpha_i} [F_{i,t} ]- {\rm{CVaR}}_{\alpha_i} [F_{i,t-1} ])^2 = (C_i(\hat{x}_t) -C_i(\hat{x}_{t-1}) )^2 \nonumber \\
    &\leq L_0^2 \left\|\hat{x}_t-\hat{x}_{t-1} \right\|^2 \leq L_0^2 \left\| x_t-x_{t-1} +\delta u_t-\delta u_{t-1} \right\|^2 \leq L_0^2(2\left\|x_t-x_{t-1}\right\|^2 + 2\left\|\delta u_t-\delta u_{t-1}\right\|^2  ) \nonumber \\
    &\leq 2L_0^2  \left\| x_t - x_{t-1}\right\|^2 +2\delta^2(2\left\|u_t \right\|^2 + 2\left\|u_{t-1} \right\|^2 ) \leq 2L_0^2  \left\| x_t - x_{t-1}\right\|^2 + 8L_0^2 N\delta^2,
\end{align}
where the last inequality is because $\left\|u_t \right\|^2 =\sum_{j=1}^N\left\|u_{i,t} \right\|^2=N $.
Recall that $x_{i,t}= \mathcal{P}_{\mathcal{X}_i^{\delta}} ( x_{i,t-1} - \eta_{i} \bar{g}_{i,t-1})$, we get that $\left\|x_{i,t}-x_{i,t-1}\right\|= \left\| \mathcal{P}_{\mathcal{X}_i^{\delta}} ( x_{i,t-1} - \eta \bar{g}_{i,t-1}) - \mathcal{P}_{\mathcal{X}_i^{\delta}} ( x_{i,t-1}) \right\| \leq \eta \left\|\bar{g}_{i,t-1}\right\| $. Substituting this into inequality (\ref{eq_term1}), we have that
\begin{align*}
    ({\rm{CVaR}}_{\alpha_i} [F_{i,t} ]- {\rm{CVaR}}_{\alpha_i} [F_{i,t-1} ])^2 
    \leq 2L_0^2 \eta^2 \left\|\bar{g}_{t-1}\right\|^2 + 8L_0^2 N\delta^2,
\end{align*}
where $\bar{g}_{t-1}$ is a concatenated vector of $\bar{g}_{i,t-1}$, $i=1\ldots,N$.
According to Lemma \ref{lemma:sum_epsihat_bound}, with probability $1-\gamma$, we have$|\hat{\varepsilon}_{i,t}|^2 \leq 
\frac{U^2}{\alpha_i^2}\frac{ \ln(2T/\gamma)}{2n_t}\leq \frac{U^2}{\alpha_i^2}\frac{ \ln(2T/\gamma)}{2bU^2(T-t+1) }\leq \frac{\ln(2T/\gamma) }{2b\alpha_i^2}$. Applying this inequality and (\ref{eq_term1}) into (\ref{eq_bar_git}), it gives that
\begin{align*}
    \left\|\bar{g}_{i,t}\right\|^2 
    & \leq \frac{d_i^2}{\delta^2}(4L_0^2\eta^2\left\| \bar{g}_{t-1}\right\|^2+16L_0^2N\delta^2 +  \frac{4\ln(2T/\gamma)}{b\alpha_i^2}  )\nonumber \\
    &\leq \frac{4d_i^2L_0^2\eta^2}{\delta^2} \left\| \bar{g}_{t-1}\right\|^2 +16d_i^2NL_0^2 +\frac{4d_i^2 \ln(2T/\gamma)}{b\alpha_i^2 \delta^2}.
\end{align*}
Summing up this inequality over $i=1,\ldots,N$, and setting $\sum_{i}\frac{1}{\alpha_i^2}=S(\alpha)$, we obtain that 
\begin{align*}
    \left\|\bar{g}_{t}\right\|^2 
    \leq \frac{4d_i^2L_0^2N \eta^2}{\delta^2} \left\| \bar{g}_{t-1}\right\|^2 +16d_i^2N^2L_0^2 +\frac{4d_i^2 \ln(2T/\gamma)S(\alpha)}{b \delta^2}.
\end{align*}
Note that $\beta=\frac{4d_i^2L_0^2N \eta^2}{\delta^2}$. Telescoping this inequality and rearranging the term, we have that
\begin{align*}
    \sum_{t=1}^T\left\| \bar{g}_{t} \right\|^2 \leq \frac{1}{1-\beta}\left\|\bar{g}_{1} \right\|^2 + \frac{16d_i^2 N^2L_0^2}{1-\beta} 
    +\frac{4d_i^2 {\rm{ln}}(2T/\gamma) S(\alpha)}{(1-\beta) b \delta^2}T.
\end{align*}
The proof ends by using $\left\| \bar{g}_{i,t} \right\| \leq \left\| \bar{g}_{t} \right\| $.
\end{proof}

\subsection*{C.2 Proof of Theorem \ref{thm:3}}
\begin{proof}
With the result in (\ref{eq:bar_git}) and following similar lines as in the proof of Lemma \ref{lemma:regret_decomp}, it can be verified that the following holds
\begin{align*}
    {\rm{R}}_{C_i}^3(T) \leq &\frac{D_x^2}{2\eta}   + \frac{\eta}{2}  \mathbb{E} \big[ \sum_{t=1}^{T} \left\|\bar{g}_{i,t}\right\|^2 \big]+ \frac{d_i D_x }{\delta}\mathbb{E} \big[ \sum_{t=1}^{T} \left\| \hat{\varepsilon}_{i,t}\right\| \big]  
    +4 \delta L_0\sqrt{N}T  + \Omega L_0 \delta T , \nonumber 
\end{align*}
where $ \hat{\varepsilon}_{i,t}$ as defined in (\ref{eq:def_hat_varep}). Applying Lemmas \ref{lemma:sum_epsihat_bound} and \ref{lemma:RF_gradient} and substituting the bounds into the inequality above, it gives that
\begin{align}
    {\rm{R}}_{C_i}^3(T) \leq &\frac{D_x^2}{2\eta}   +\frac{\left\|\bar{g}_{1} \right\|^2 \eta }{2(1-\beta)} + \frac{8d_i^2 N^2L_0^2\eta}{1-\beta} 
    +\frac{2d_i^2 {\rm{ln}}(2T/\gamma) S(\alpha) \eta}{(1-\beta) b \delta^2}T   \nonumber \\
    &+ \frac{d_i D_x \sqrt{2 \ln(2T/\gamma)}}{\alpha_i \delta\sqrt{b}}T^{1-\frac{a}{2}} +4 \delta L_0\sqrt{N}T  + \Omega L_0 \delta T . \nonumber
\end{align}
Recall that $\eta=\frac{D_x  }{d_i L_0 N}T^{-\frac{3a}{4}}$,  $\delta=\frac{D_x }{ N^{\frac{1}{6}}}T^{-\frac{a}{4}}$, we have $\beta=\frac{4d_i^2L_0^2N \eta^2}{\delta^2}=4N^{\frac{2}{3}}  T^{-a}\leq \frac{1}{2}$ when $T\geq (8N^{\frac{2}{3}} )^{\frac{1}{a}}$. Therefore, we have $\frac{1}{1-\beta}\leq 2$, and it follows that
\begin{align*}
    {\rm{R}}_{C_i}^3(T) \leq &\frac{D_x^2}{2\eta}   +\left\|\bar{g}_{1} \right\|^2 \eta + 16d_i^2 N^2L_0^2\eta +\frac{4d_i^2 {\rm{ln}}(2T/\gamma) S(\alpha)\eta}{b \delta^2}T   \nonumber \\
    &+ \frac{d_i D_x \sqrt{2 \ln(2T/\gamma)}}{\alpha_i \delta\sqrt{b}}T^{1-\frac{a}{2}} +4 \delta L_0\sqrt{N}T  + \Omega L_0 \delta T  .\nonumber
\end{align*}
Substituting $\delta$ and $\eta$ into this inequality and we can get $ {\rm{R}}^3_{C_i}(T) = \mathcal{O}( D_x d_i L_0 NS(\alpha) \ln(T/\gamma) T^{1-\frac{a}{4}} )$. The proof is complete.
\end{proof}

\section*{D. Additional Numerical Experiments}
We provide additional simulation results for different sampling strategies related to different confidence levels $\gamma$. Two different sampling strategies are shown in Figure \ref{append:fig:samples}. Figure \ref{append:fig:cost} shows that Algorithm \ref{alg:algorithm1} converges faster if more samples are collected and large sample size can decrease the variance of the algorithm.  
\begin{figure}[h]
	\centering
	\subfigure{\includegraphics[scale=0.45]{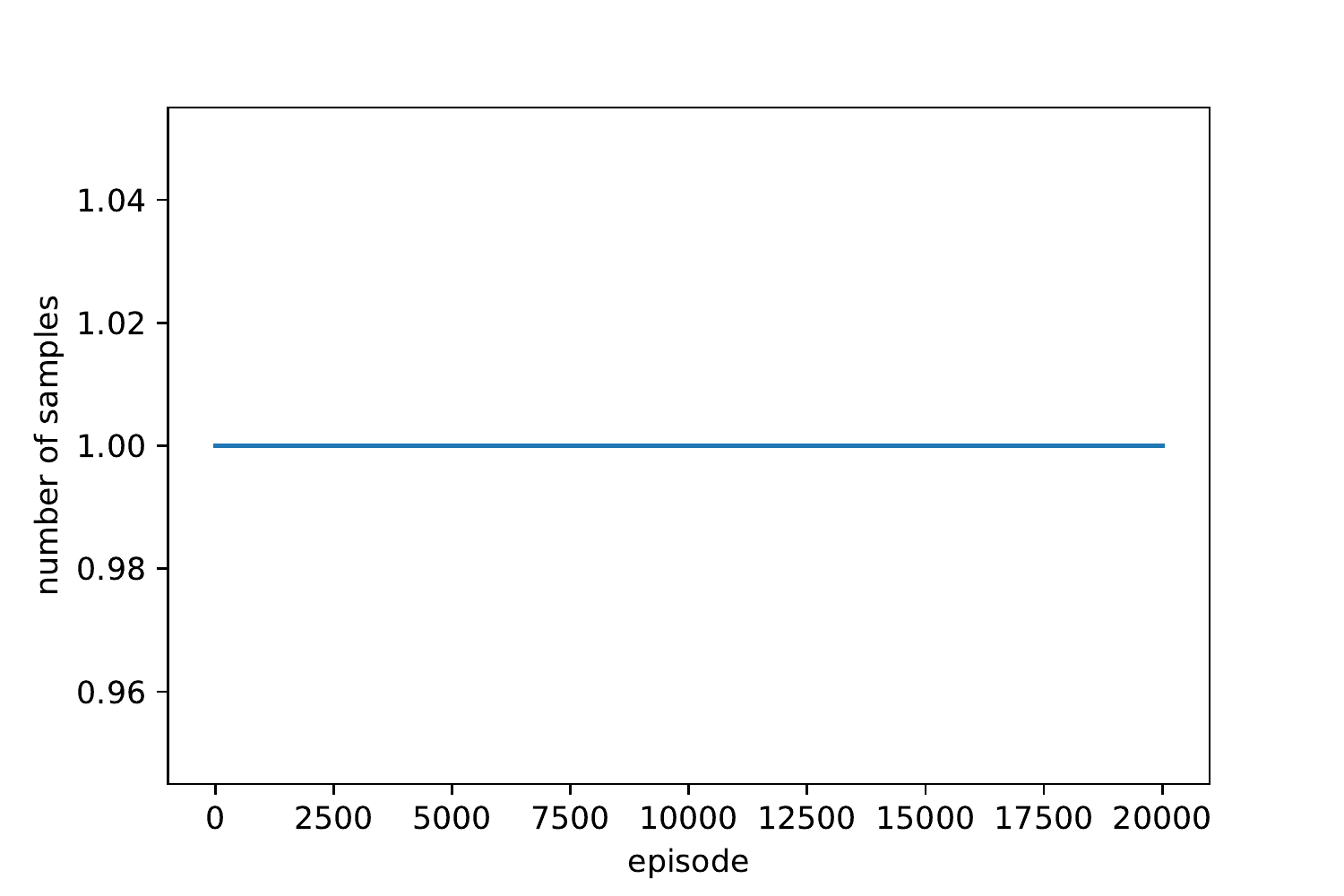}} \quad
	\subfigure{\includegraphics[scale=0.45]{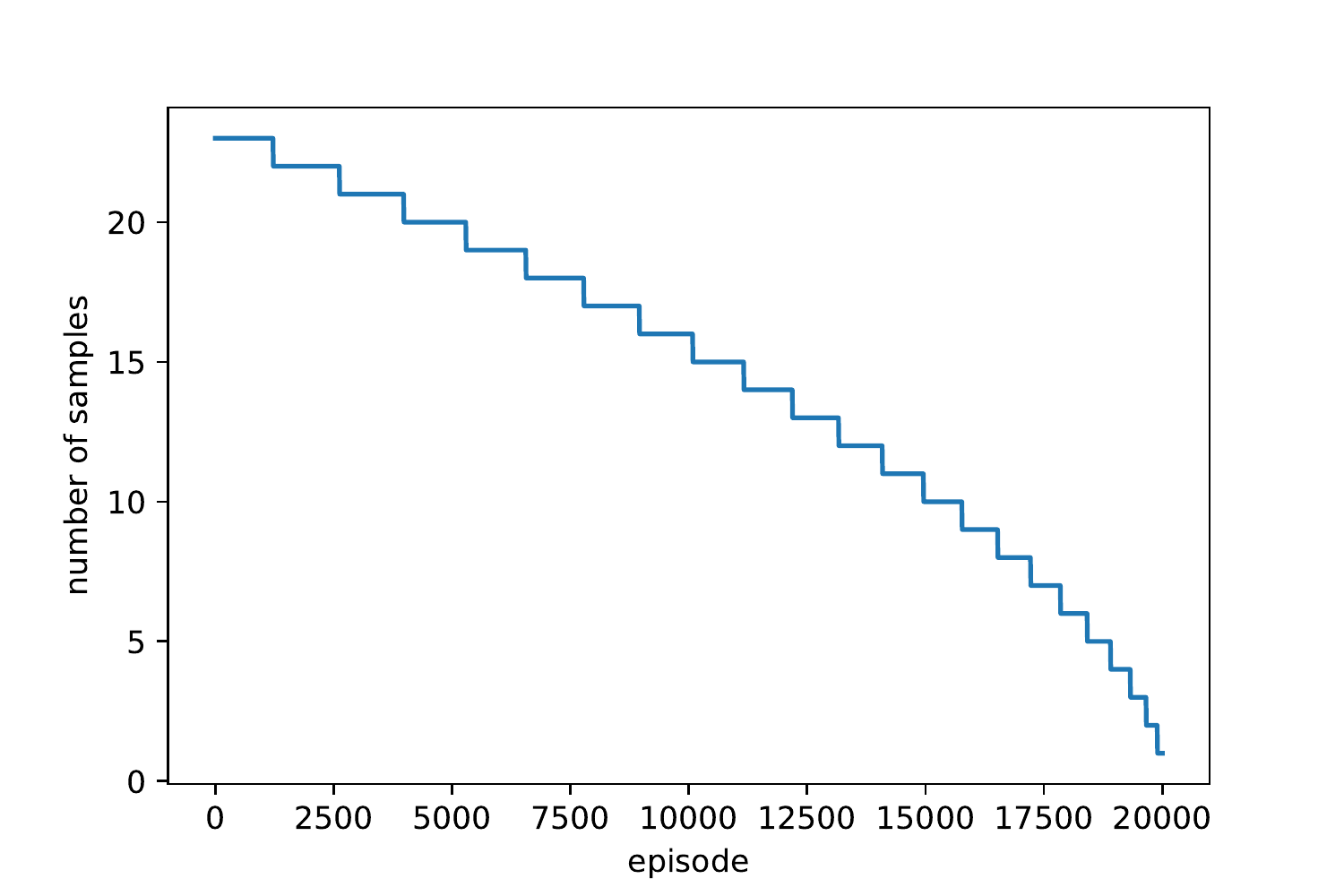}}
	\caption{The different choices of number of samples of Algorithm \ref{alg:algorithm1}.}
	\label{append:fig:samples}
	\vspace{-4mm}
\end{figure}

Moreover, since risk-neural learning is a special case of risk-averse learning by selecting $\alpha_i=1$, $i=1,\ldots,N$, our risk-averse learning algorithms can be also used for risk-neutral games.
Although in risk-averse games it is typically hard to calculate the Nash equilibrium, in risk-neutral games calculating the Nash equilibrium is possible. We define the expected cost function of agent $i$ as $u_i=\mathbb{E}_{\xi_i}[J_i]=\mathbb{E}_{\xi_i}[-(2-\sum_jx_j)x_i+0.1x_i+\xi_i x_i+1] = -(2-\sum_jx_j)x_i+0.6x_i+ 1$, where $\xi_i\sim U(0,1)$. With this expected cost function, this risk-neural game becomes a convex and monotone game. For this class of games, \cite{rosen1965existence} has shown that a unique Nash equilibrium exists. Specifically, setting the  gradients of the expected cost function equal to 0, we can calculate the Nash equilibrium as $(0.467,0.467)$. Figure \ref{fig_ne} shows that our algorithm converges to a neighborhood of the Nash equilibrium. This experiment verifies the correctness of our methodology.
\begin{figure}[t]
	\centering
	\subfigure{\includegraphics[scale=0.45]{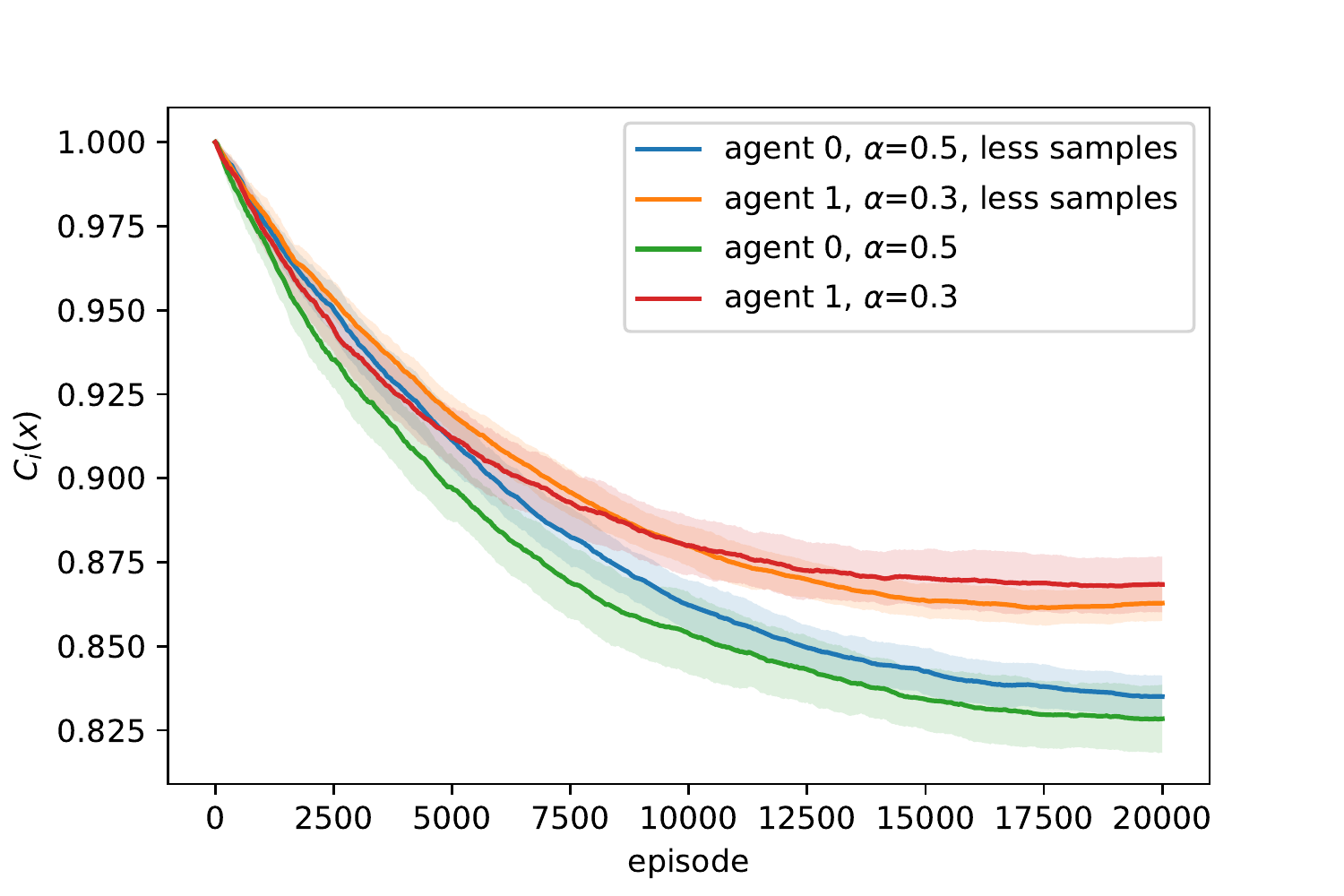}} \quad
	\subfigure{\includegraphics[scale=0.45]{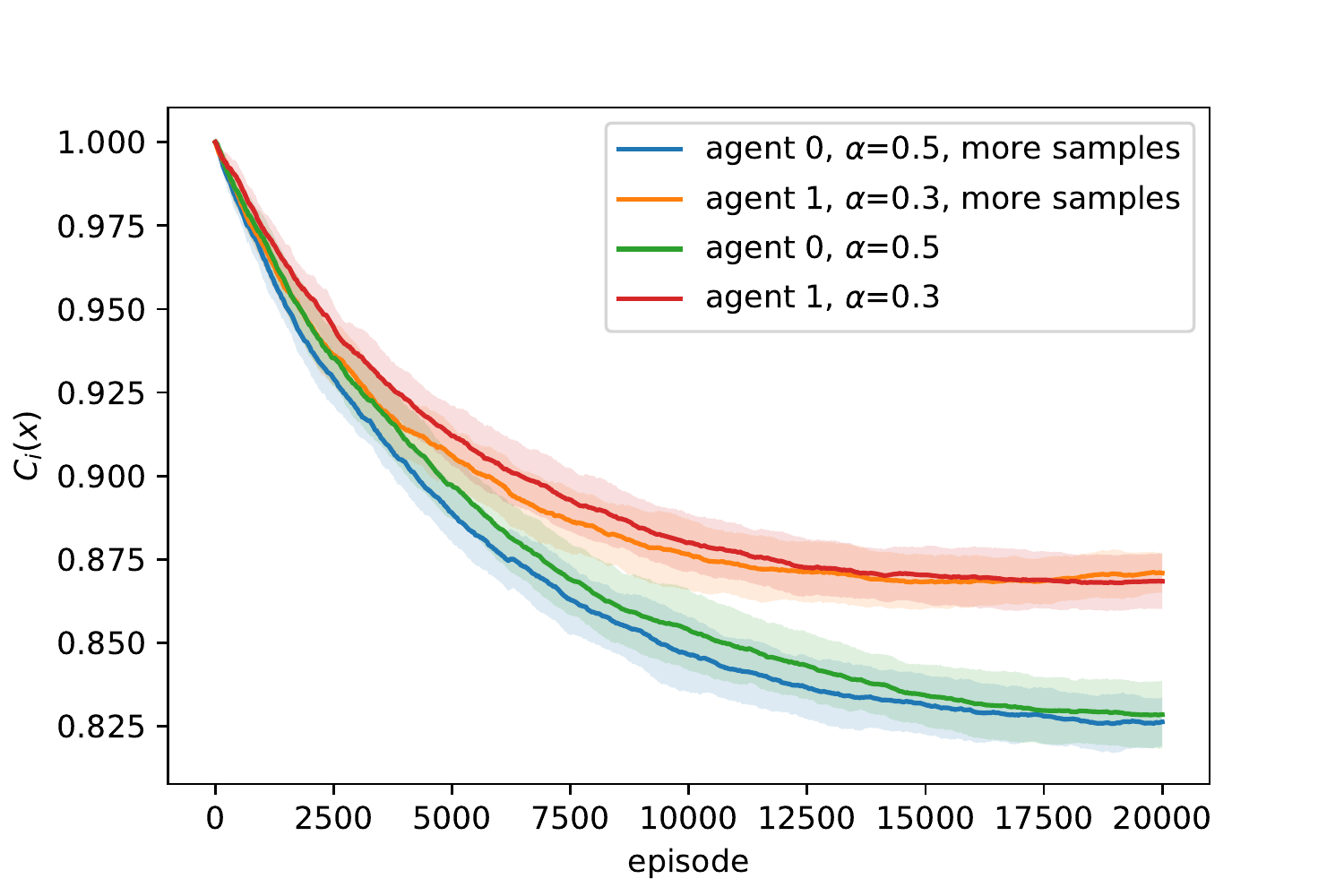}}
	\caption{CVaR values achieved by Algorithm \ref{alg:algorithm1} with different choices of samples. The solid lines and shades are averages and standard deviations over 20 runs.}
	\label{append:fig:cost}
	\vspace{-4mm}
\end{figure}
\begin{figure}[h]
\begin{center}
\centerline{\includegraphics[width=0.45\columnwidth]{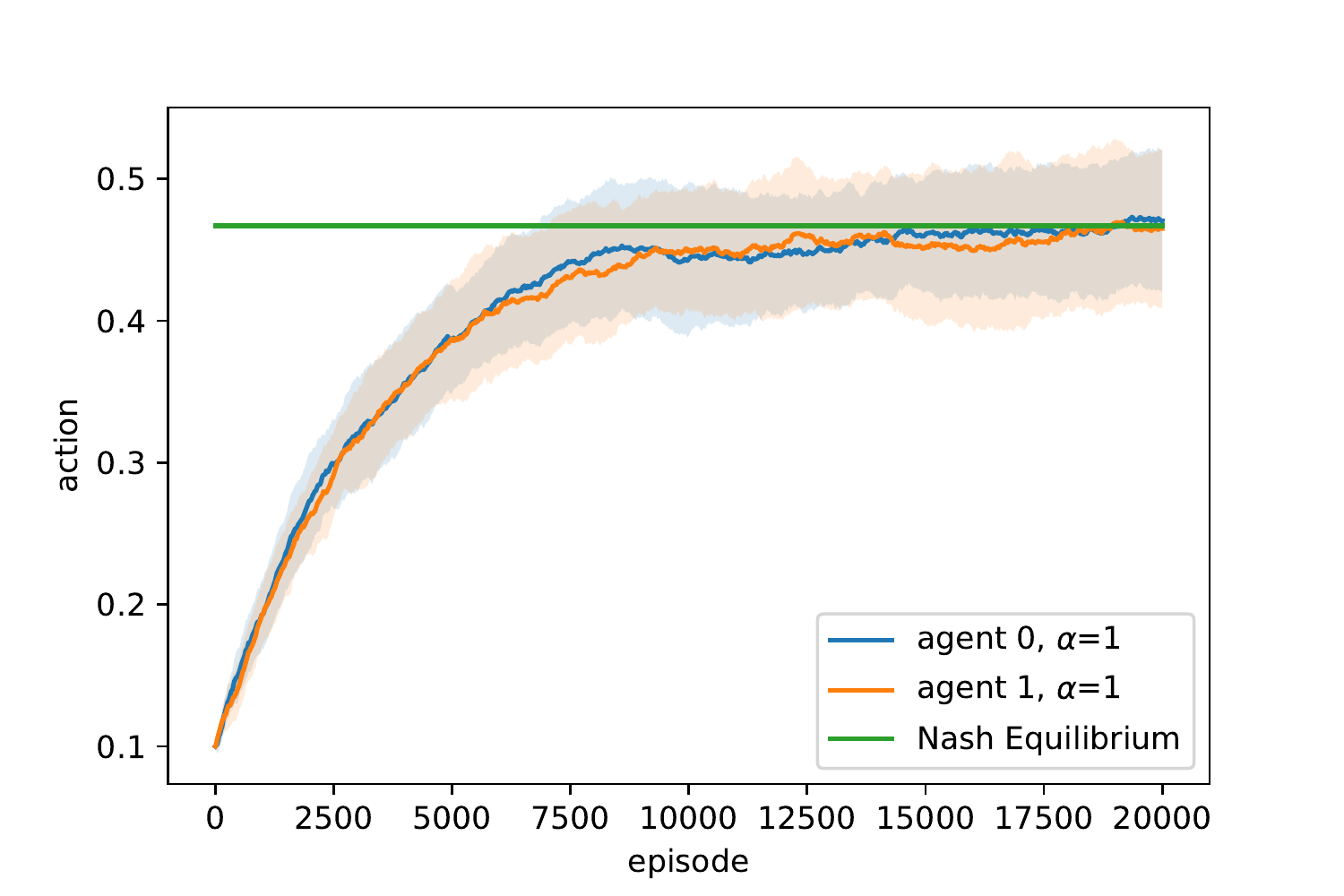}}
\caption{Action values for risk-neural agents.}
\end{center}
\label{fig_ne}
\vskip -0.4in
\end{figure}


\end{document}